\documentclass{article}

\usepackage[utf8]{inputenc}

\usepackage[colorinlistoftodos, textwidth=22mm, shadow]{todonotes}

\usepackage{hyperref}
\definecolor{pear}{HTML}{D64141}
\definecolor{pearThree}{HTML}{FF69B4}
\definecolor{pearDark}{HTML}{2980B9}
\definecolor{pearDarker}{HTML}{F330DB}
\hypersetup{
	colorlinks,
	citecolor=pearDark,
	linkcolor=pear,
	urlcolor=pearDarker}

\usepackage{amsfonts, amsmath, amssymb,amsthm}
\usepackage{xfrac}
\usepackage{graphicx, color}
\usepackage{natbib}
\usepackage{dsfont}
\usepackage{bbm}
\usepackage[algo2e]{algorithm2e}
\usepackage{algorithm}

\usepackage{import}
\usepackage{enumitem}
\setitemize{noitemsep,topsep=0pt,parsep=0pt,partopsep=0pt}
\usepackage{mathtools}
\usepackage{svg}
\usepackage{footnote}
\usepackage{tablefootnote}

\usepackage{pdflscape}
\usepackage{booktabs, multirow}
\usepackage[flushleft]{threeparttable}
\usepackage{afterpage}
\usepackage[margin=1in]{geometry}

\newtheorem{theorem}{Theorem}
\newtheorem{lemma}{Lemma} 
\newtheorem{proposition}{Proposition} 
\newtheorem{remark}[theorem]{Remark}
\newtheorem{corollary}[theorem]{Corollary}
\newtheorem{definition}[theorem]{Definition}

\renewcommand{\widehat}{\hat}

\DeclareMathOperator{\kl}{kl}
\DeclareMathOperator{\sign}{sign}
\DeclareMathOperator{\EE}{\mathbb{E}}
\DeclareMathOperator{\PP}{\mathbb{P}}
\DeclareMathOperator{\R}{\mathbb{R}}

\DeclareMathOperator{\Ng}{\mathcal{N}}
\DeclareMathOperator*{\argmax}{arg\,max}
\DeclareMathOperator*{\argmin}{arg\,min}

\DeclareMathOperator{\median}{\mathrm{Median}}

\DeclareMathOperator{\loglog}{\log\!\log}

\newcommand*\dIff{\mathop{}\!\mathrm{d}}
\renewcommand{\d}{\mathrm{d}}
\newcommand{\hmu}{\widehat{\mu}}

\renewcommand{\epsilon}{\varepsilon}
\newcommand{\ber}{\mathrm{Ber}}

\newcommand{\ind}{\mathds{1}}

\newcommand{\cF}{\mathcal{F}}
\newcommand{\cB}{\mathcal{B}}

\DeclarePairedDelimiter\floor{\lfloor}{\rfloor}

\usepackage{accents}
\renewcommand{\epsilon}{\varepsilon}
\newcommand{\hepsilon}{\widehat{\epsilon}}
\newcommand{\unu}{{\underaccent{\bar}{\nu}}}
\newcommand{\bR}{\bar{R}}
\newcommand{\sT}{\hat{T}}
\newcommand{\bxi}{\bar{\xi}}
\newcommand{\hQ}{\widehat{Q}}
\newcommand{\hk}{\widehat{k}}

\newcommand{\hl}{\hat{l}}
\newcommand{\hr}{\hat{r}}

\newcommand{\tQ}{\widetilde{Q}}

\newcommand{\tbp}{\textit{TBP}\xspace} 
\newcommand{\stbp}{\textit{MTBP}\xspace} 
\newcommand{\utbp}{\textit{UTBP}\xspace} 
\newcommand{\ctbp}{\textit{CTBP}\xspace} 
\newcommand{\OR}{\mathrm{OR\;}}

\usepackage{xspace}
\newcommand{\unif}{\hyperref[alg:unif]{\texttt{Uniform}}\xspace}
\newcommand{\explore}{\hyperref[alg:explore]{\texttt{Explore}}\xspace}
\renewcommand{\choose}{\hyperref[alg:choose]{\texttt{Choose}}\xspace}
\newcommand{\ctb}{\hyperref[alg:CTB]{\texttt{CTB}}\xspace}
\newcommand{\stb}{\hyperref[alg:STB2]{\texttt{MTB}}\xspace}
\newcommand{\dstb}{\hyperref[rem:DSTB]{\texttt{DEC-MTB}}\xspace}
\renewcommand{\root}{\texttt{root}\xspace}

\newcommand{\B}{\cB} 
\newcommand{\Bs}{\cB_m} 
\newcommand{\Bu}{\cB_u} 
\newcommand{\Bc}{\cB_c}

\newcommand{\CTB}{\hyperref[alg:CTB]{\texttt{CTB}}\xspace}
\newcommand{\UTB}{\hyperref[alg:UTB]{\texttt{UTB}}\xspace}

\newcommand{\SR}{SR\xspace}

\begin{document}

\title{The Influence of Shape Constraints on the Thresholding Bandit Problem}

\author{James Cheshire \\Otto von Guericke
University Magdeburg\\  james.cheshire@ovgu.de
   \and Pierre Menard \\Centre Inria Lille - Nord Europe \\ pierre.menard@inria.fr
   \and Alexandra Carpentier\\Otto von Guericke
University Magdeburg\\alexandra.carpentier@ovgu.de }
\maketitle

\begin{abstract}

    We investigate the stochastic \emph{Thresholding Bandit problem} (\tbp) under several \emph{shape constraints}. On top of (i) the vanilla, unstructured \tbp, we consider the case where (ii) the sequence of arm's means $(\mu_k)_k$ is monotonically increasing \stbp, (iii) the case where $(\mu_k)_k$ is unimodal \utbp and (iv) the case where $(\mu_k)_k$ is concave \ctbp. In the \tbp problem the aim is to output, at the end of the sequential game, the set of arms whose means are above a given threshold. The regret is the highest gap between a misclassified arm and the threshold. 
    In the fixed budget setting, we provide \emph{problem independent} minimax rates for the expected regret in all settings, as well as associated algorithms. 
    We prove that the minimax rates for the regret are (i) $\sqrt{\log(K)K/T}$ for \tbp, (ii) $\sqrt{\log(K)/T}$ for \stbp, (iii) $\sqrt{K/T}$ for \utbp and (iv) $\sqrt{\log\log K/T}$ for \ctbp, where $K$ is the number of arms and $T$ is the budget. These rates demonstrate that \textit{the dependence on $K$} of the minimax regret varies significantly depending on the shape constraint. This highlights the fact that the shape constraints modify fundamentally the nature of the \tbp.
\end{abstract}

\section{Introduction}\label{sec:intro}

Stochastic multi-armed bandit problems consider situations in which a learner faces multiple unknown probability distributions, or ``arms'', and has to sequentially sample these arms.
In this paper, we focus on the Thresholding Bandit Problem (\tbp), a \textit{Combinatorial Pure Exploration (CPE)} bandit setting introduced by \citet{chen2014combinatorial}. The learner is presented with $[K] = \{1,\ldots,K\}$ arms, each following an unknown distribution $\nu_k$ with unknown mean $\mu_k$.  Given a budget $T>0$, the learner samples the arms sequentially for a total of $T$ times and then aims at predicting the set of arms whose mean is above a given threshold $\tau\in\R$. 

The performance of the learner is measured through the \emph{expected simple regret} which in this setting is the expected maximal gap between $\tau$ and the mean of a misclassified arm. Note that our problem is in fact akin to estimating in a sequential setting a given level-set of a discrete function under shape constraints.

In this paper we will be interested only in the \emph{problem independent} case, and want to characterise the \emph{minimax-order} of the expected simple regret on various sets of bandit problems. In particular we study the influence of various \textit{shape constraints on the sequence of means of the arms}, on the \tbp problem, i.e.~see how classical shape constraints influence the minimax rate of the expected simple regret. We will consider four shape constraints.

\paragraph{Vanilla, unstructured case \tbp}
First we consider the vanilla case where we only assume that the distributions of the arms are supported in $[0,1]$. We will refer to this case as the unstructured problem, (\tbp). The fixed confidence version of \tbp was studied in~\cite{chen2014combinatorial, chen2016combinatorial} -see also e.g.~\cite{even2002pac, chen2015optimal,simchowitz2017simulator,garivier2016optimal} for papers in the related best arm identification and  TOP-M setting\footnote{In the TOP-M setting, the objective of the learner is to output the $M$ arms with highest means. A popular version of it it is the TOP-1 problem where the aim is to find the arm that realises the maximum.} in the fixed confidence case. The fixed budget version of \tbp was studied in~\citet{chen2014combinatorial, locatelli2016optimal, mukherjee2017thresholding}, \citet{JieZhong17bandits} 
- but also see e.g.~\cite{bubeck2009pure, audibert2010best, gabillon2012best, carpentier2016tight} for papers in the related best arm identification and TOP-M setting in the fixed budget case. These papers almost exclusively concern the \textit{problem dependent regime}, which is not the focus of this paper, and the adaptation of their rate to the problem independent case is sub-optimal, see the discussion under Theorem~\ref{thm:unstr_minmax} for a more thorough comparison to this literature, and Appendix~\ref{app:litreview} for details.\\
In this paper, we prove that the minimax-optimal order of the expected simple regret in \tbp is $\sqrt{\frac{K\log K}{T}}$. While a simple uniform-sampling strategy attains this bound, the lower bound is more interesting, in particular the presence of the $\sqrt{\log K}$ term. See the discussion following Theorem~\ref{thm:unstr_minmax}. For a discussion on the performance of the uniform-sampling strategy in the \textit{problem dependent} regime, see Appendix \ref{app:litreview}.


\paragraph{Monotone constraint, \stbp.} We then consider the problem where on top of assuming that the distributions are supported in $[0,1]$, we assume that the sequence of means $(\mu_k)_k$ is monotone - this is problem \stbp. This specific instance of the \tbp is introduced within the context of drug dosing in \citet{garivier2017thresholding}. In this paper, the authors provide an algorithm for the fixed confidence setting that is optimal from a \emph{problem dependent} point of view. 
The shape constraint on the means of the arms implies that the \stbp is related to \textit{noisy binary search}, i.e.~inserting an element into its correct place within an ordered list when only noisy labels of the elements are observed, see~\cite{feige1994computing}. In the noiseless case, an effective approach due to the shape constraint is to conduct a binary search - and the classification of the arms can therefore be performed in just $O(\log(K))$ steps, while $K$ steps are needed in the noiseless \tbp. It is therefore clear that \stbp is radically different from \tbp, even in the noiseless case. In the noisy case, the learner has to sample many times each arm in order to get a reliable decision at each step. While a simple naive strategy, although sufficient in \citet{Xu19binaryduel}, is to do \textit{noisy binary search} where at each step the learner simply samples about $O(T/\log(K))$ times each arm, there are clear hints from the literature that in the \stbp this is not going to be optimal. For the related yet different problem of noisy binary search, \cite{feige1994computing}, \citet{ben2008bayesian} and \citet{emamjomeh2016deterministic} solve this issue by introducing a noisy binary search \emph{with corrections} - see also \cite{Nowak09binary}, \cite{karp2007noisy}. However, all these papers consider the problem of noisy binary search in settings with more structural assumptions and where the objective is more related to a fixed confidence setting, their results are therefore not directly applicable to our setting. See the discussion under Theorem~\ref{thm:mon_minmax} for a more thorough comparison to this literature and see  Appendix~\ref{app:litreview} for details.  \\
In this paper, we prove that the minimax-optimal order of expected simple regret in \stbp is $\sqrt{\log (K)/T}$. Interestingly and as highlighted in this paragraph, this rate is much smaller than the minimax rate over \tbp. This reflects the fact that the monotone shape constraint makes the problem much simpler than \tbp, and closer to noisy binary search. Further discussion on the comparison between the \tbp and \stbp, specifically the difference coming from the monotone assumption, can be found in Appendix~\ref{app:litreview} and see the algorithm \explore  and the associated text in Section \ref{sec:optimal_alg} for more intuition on the link to noisy binary search. Discussion on the performance of our algorithms for the \stbp in the \textit{problem dependent} regime can also be found in Appendix \ref{app:litreview}.

\paragraph{Unimodal constraint, \utbp.} We also consider the problem where on top of assuming that the distributions are supported in $[0,1]$, we assume that the sequence of means $(\mu_k)_k$ is unimodal - this is problem \utbp. It has not been considered to the best of our knowledge. However similar problems have been studied such that identifying the best arm or minimizing the cumulative regret \cite{combes2014unimodal, combes2014unimodal2, paladino2017unimodal, yu2011unimodal}. \cite{ paladino2017unimodal, combes2014unimodal2} focus on the \textit{problem dependent regime}, and are not transferable - at least to the best of our knowledge - to the problem independent setting. \cite{yu2011unimodal, combes2014unimodal} are closer to our problem as it focuses on the problem independent regime. However, they consider the $\mathcal X$-armed setting (continuous set of arms e.g.~in $[0,1]$) setting and assume H\"older type regularity assumption around the maximum, which prevents jumps in the mean vector. These results therefore do not apply to our setting, where of course jumps are bound to happen as we are in the discrete setting. See the discussion under Theorem~\ref{thm:un_minmax} for a more thorough comparison to this literature.\\
In this paper, we prove that the minimax-optimal order of the expected simple regret in \utbp is of order $\sqrt{K/T}$. This is interesting in contrast to the rate of \stbp. Monotone bandit problems are much easier than unimodal bandit problems - which can be written as a combination of a non-decreasing bandit problem, and a non-increasing bandit problem. This is however not very surprising, as finding the maximum of the unimodal bandit problem - i.e.~the points where the non-increasing and non-decreasing bandit problems merge - is difficult.

\paragraph{Concave constraint, \ctbp.} Finally we consider the problem where on top of assuming that the distributions are supported in $[0,1]$, we assume that the sequence of means $(\mu_k)_k$ is concave - this is problem \ctbp. To the best of our knowledge this setting has not yet been consider in the literature. However, two related problems have been considered: the problem of estimating a concave function, and the problem of optimising a concave function - for both problems, mostly in the continuous setting, which renders a comparison with our setting delicate. The problem of estimating a concave function has been thoroughly studied in the noiseless setting, and also in the noisy setting, see e.g.~\cite{simchowitz2018adaptive}, where the setting of a continuous set of arms is considered, under H\"older smoothness assumptions. The problem of optimising a convex function in noise without access to its derivative - namely zeroth order noisy optimisation - has also been extensively studied. See e.g.~\cite{yudin}[Chapter 9], and~\cite{wang2017stochastic, agarwal2011stochastic, liang2014zeroth} to name a few, all of them in a continuous setting and in dimension $d$. The focus of this literature is however very different than ours, as the main difficulty under their assumption is to obtain a good dependence in the dimension $d$, and in this setting logarithmic factors are not very relevant. See the discussion under Theorem~\ref{thm:con_minmax} for a more thorough comparison to this literature.\\
In this paper, we prove that the minimax-optimal order of the expected simple regret in \ctbp is $\sqrt{\loglog(K)/T}$. This is interesting in contrast to rate in the case of \utbp. Concave bandit problems are much easier than unimodal bandit problems. Also, if we compare with \stbp, we have that concave bandit problems are also much easier than monotone bandit problems, which is perhaps surprising - in particular the fact that the dependence in $K$ is much smaller.

\begin{table}
\label{tab:KV}
\def\arraystretch{1.85}
 \begin{tabular}{c||c|c|c|c}
  Results & Unstructured \tbp  & Monotone \tbp & Unimodal \tbp & Concave \tbp  \\ \hline \hline
 Regret & $\sqrt{\frac{K\log K}{T}}$ & $\sqrt{\frac{\log K }{T}}$ & $\sqrt{\frac{K}{T}}$ & $\sqrt{\frac{\loglog K}{T}} $
 \end{tabular}
 \caption{Order of the minimax expected simple regret for the thresholding bandit problem, in the case of all four structural assumptions on the means of the arms considered in this paper. All results are given up to universal multiplicative constants.}
 \label{tab:KV}
 \end{table}

\paragraph{Organisation of the paper}

Our results are summarized in Table~\ref{tab:KV}. See also Appendix~\ref{sec:contin} for an adaptation of these results in the $\mathcal X$-armed bandit setting. In Section~\ref{sec:prob} we define the setting and the \tbp, \stbp, \ctbp and \utbp problems. 
Minimax rates for the expected regret for all cases are given in Section \ref{sec:minmaxreg}.  In Section \ref{sec:optimal_alg} we describe algorithms attaining the minimax rates of Section \ref{sec:minmaxreg}, again for all cases. The Appendix contains the proofs for all results, as well as formulation of the upper and lower bounds leading to the minimax rates in a broader setting, transposition of some results in the fixed confidence setting, and also some additional discussions and remarks.

\vspace{-0.3cm}
\section{Problem formulation}\label{sec:prob}

The learner is presented with a $K$-armed bandit problem $\unu =\{\nu_1,\ldots,\nu_K\}$, with $K\geq 3$, where $\nu_k$ is the unknown distribution of arm $k$. Let $\tau\in \mathbb R$ be a fixed threshold known to the learner.  
We aim to devise an algorithm which classifies arms as above or below threshold $\tau$. That is, the learner aims at finding the vector $Q\in\{-1,1\}^K$ that encodes the true classification, i.e. $Q_k = 2\ind_{\{\mu_k \geq \tau \}}-1$  with the convention $Q_k = 1$ if arm $k$ is above the threshold and $Q_k = -1$ otherwise.

The \emph{fixed budget} bandit sequential learning setting goes as follows: the learner has a budget $T>0$ and at each round $t \leq T$, the learner pulls an arm $k_t\in [1,K]$ and observes a sample $Y_{t}\sim \nu_{k_t}$, conditionally independent from the past.
After interacting with the bandit problem and expending their budget, the learner outputs a vector $\hQ\in\{-1,1\}^K$ and the aim is that it matches the unknown vector $Q$ as well as possible.

That is, the \textit{fixed budget} objective of the learner following the strategy $\pi$ is then to minimize the expected simple regret of this classification for $\hat Q:= \hat Q^\pi$:
\[\bR_{T}^{\unu,\pi} = \EE_{\unu} \!\left[\max_{\{k\in[K]:\ \hQ_k^\pi \neq Q_k\}} \Delta_k \right],\] 
where $\Delta_k : = |\tau - \mu_k|$ is the gap of arm $k$, and where $\EE_{\unu}$ is defined as the expectation on problem $\unu$ and $\mathbb P_{\unu}$ the probability.  
We also write for the simple regret as a random variable $R_{T}^{\unu,\pi} = \max_{\{k\in[K]:\ \hQ_k^\pi \neq Q_k\}} \Delta_k  \,.$ 
When it is clear from the context we will remove the dependence on the bandit problem $\unu$ and/or the strategy $\pi$. We now present several sets of bandit problems that correspond to our four shape constraints.
\paragraph{Vanilla, unstructured case \tbp} We assume that the distribution of all the arms $\nu_k$ are supported in $[0,1]$. 
We denote by $\mu_k$ the mean or arm $k$. Let $\cB:= \cB( K)$ be the set of such problems. 

\paragraph{Monotone case \stbp} 

We denote by $\Bs$ the set of bandit problems,
$$\Bs :=  \{\nu \in \cB:\ \mu_1 \leq\mu_2 \leq  \ldots \leq\mu_K\}\;,$$
where the learner is given the  additional information that the sequence of means $\left( \mu_k \right)_{k \in [K]}$ is a monotonically increasing sequence.

\paragraph{Unimodal case \utbp}
We will denote by $\Bu$ the set of bandit problems, 
$$\Bu :=   \left\{\nu \in \cB:\ \exists k^*\in[K] \text{s.t.} \forall l \leq k^*, \mu_{l-1} \leq \mu_l \; \text{and} \; \forall l \geq k^*, \mu_l \geq \mu_{l+1} \right\}\;,$$
where the learner is given the  additional information that the sequence of means $\left( \mu_k \right)_{k \in [K]}$ is unimodal. 

\paragraph{Concave case \ctbp}
We will denote by $\Bc$ the set of bandit problems, 
\[\Bc := \left\{ \nu \in \cB: \forall 1<k< K-1, \frac{1}{2}\mu_{k-1} + \frac{1}{2}\mu_{k+1} \leq \mu_k\right\}\,, \]
where the learner is given the  additional information that the sequence of means $\left( \mu_k \right)_{k \in [K]}$ is concave.

\paragraph{Minimax expected regret} Consider a set of bandit problems $\tilde{\mathcal B}$ - e.g.~$\Bu, \Bs, \Bc, \cB$. The minimax optimal expected regret on $\tilde B$ is then
$$\bR_T^*(\tilde{\mathcal B}) := \inf_{\pi~{\mathrm{strategy}}} \sup_{\unu \in \tilde{\mathcal B}}\bR_{T}^{\unu,\pi}\,.$$
\vspace{-0.6cm}

\vspace{-0.3cm}
\section{Minimax expected regret for \tbp, \stbp, \utbp, \ctbp}\label{sec:minmaxreg}

In this section we present all minimax rates on the expected regret in the case of all four shape constraints.

Algorithms achieving these mini-max rates are described in Section \ref{sec:optimal_alg}. For two positive sequences of real numbers $(a_n)_n, (b_n)_n$, we write $a_n \asymp b_n$ if there exists two \textit{universal constants}\footnote{In particular, independent of $T,K$.} $0<c<C$ such that $ca_n \leq b_n \leq C a_n$.

Theorem \ref{thm:unstr_minmax} provides the minimax rate of the \tbp. The proof can be found in Appendix~\ref{app:proof_of_unstr}, i.e.~Proposition \ref{prop:unstr_lo}, and Proposition \ref{prop:unstr_exp_up}. 
\begin{theorem}\label{thm:unstr_minmax}
It holds that
$$\bR_T^*(\cB) \asymp  \sqrt{\frac{K\log K}{T}}.$$
The algorithm \unif described in Sections~\ref{sec:optimal_alg} (see also Appendix~\ref{app:proof_of_unstr}) attains this rate.
\end{theorem}
It is difficult to compare this result to state of the art literature as existing papers consider almost exclusively the \textit{problem dependent regime}, and often the fixed confidence setting. One can however deduce from~\cite{locatelli2016optimal} an upper bound of order $\sqrt{K\log(K \log T / \delta)/T}$, and from~\cite{chen2016combinatorial} a lower bound of order $\sqrt{K/T}$, which are both slightly sub-optimal.

Theorem \ref{thm:mon_minmax} provides the minimax rate of the \stbp. The proof can be found in Appendix~\ref{app:mon_minimax}, i.e.~Proposition~\ref{prop:mon_lo}, and Corollary~\ref{cor:mo_exp_up}.
\begin{theorem}\label{thm:mon_minmax}
It holds that
$$\bR_T^*(\Bs) \asymp  \sqrt{\frac{\log K}{T}}.$$
The algorithm \stb described in Section \ref{sec:optimal_alg} attains this rate.

\end{theorem}

The literature that achieves results closest to this theorem is the noisy binary search literature cited in the introduction. The results that are most comparable to ours are the ones in \citet{karp2007noisy}. They consider the special case where all arms follow a Bernoulli distribution with parameter $p_k$ and $p_1 < ... < p_K$, and the aim is to find a $i$ such that $p_i$ is close to $1/2$. In the \textit{fixed confidence setting}, they prove that the naive binary search approach is not optimal and propose an involved exponential weight algorithm, as well as a random walk binary search, for solving the problem. They prove that for a fixed $\epsilon,\delta>0$, the algorithm returns all arms above threshold with probability larger than $1-\delta$, and tolerance $\epsilon$, in an expected number of pulls less than a multiplicative constant \textit{that depends on $\delta$ in a non-specified way} times $\log_2(K)/\epsilon^2$. They prove that this is optimal up to a constant depending on $\delta$. In the paper~\cite{ben2008bayesian} they refine the dependence on $\delta$ in a slightly different setting - where one has a fixed error probability. They prove that \textit{up to terms that are negligible with respect to $\log(K)/\epsilon^2$}, a lower bound in the expected stopping time is of order $(1-\delta)\log(K)/\epsilon^2$. Even after a non-trivial transposition effort from their setting to ours, these results would still provide sub-optimal bounds in our setting as we consider the \textit{expected} simple regret - and a sharper dependence in their $\delta$ would be absolutely necessary here in all regimes to get our results.

Theorem \ref{thm:un_minmax} provides the minimax rate of the \utbp. The proof can be found in Appendix~\ref{app:un_minimax}, i.e.~Proposition~\ref{prop:unimod_lo}, and Proposition~\ref{prop:exp_uni_hi}.
\begin{theorem}\label{thm:un_minmax}
It holds that
$$\bR_T^*(\Bu) \asymp  \sqrt{\frac{K}{T}}.$$
The algorithm \UTB described in Section \ref{sec:optimal_alg} attains this rate.

\end{theorem}
Most related papers consider the problem dependent setting. However the papers~\cite{yu2011unimodal, combes2014unimodal} consider the problem independent regime, in the $\mathcal X$-armed setting and in both cases under additional shape constraint assumptions inducing that the maximum is not too "peaky" and isolated. They prove that the minimax simple regret for the TOP-1 problem is of order $\sqrt{\log(T)/T}$.\\
This seems to contradict our results, to which a direct corollary is that the minimax expected regret for finding a given level set of a $\beta$-H\"older, unimodal function in $[0,1]$ is $n^{-\frac{\beta}{2\beta+1}}$, see Appendix~\ref{sec:contin}. 
This might seem unintuitive when compared to their result where the rate is much faster. But is not, as the assumption that both papers make essentially imply that the set of arms that are $\epsilon$-close to the arm with highest mean decays in a regular way, which implies that a binary search will provide good results in this case - unlike in our setting. 

Therefore their setting is closer in essence to the \stbp problem than to the \tbp problem, as binary-search type methods work well there as highlighted in~\cite{combes2014unimodal}. And interestingly, a direct corollary to Theorem~\ref{thm:mon_minmax} for \stb is that the minimax expected regret for finding a given level set of a $\beta$-H\"older, monotone function in $[0,1]$ is $\sqrt{\log( T)/T}$, see Appendix~\ref{sec:contin}, which is very much aligned with the findings in~\cite{combes2014unimodal}.

Theorem \ref{thm:con_minmax} provides the minimax rate of the \ctbp. The proof can be found in Appendix~\ref{app:con_minimax}, i.e.~Proposition~\ref{prop:lb_cvx}, and Proposition~\ref{prop:exp_con_hi}.
\begin{theorem}\label{thm:con_minmax}
It holds that
$$\bR_T^*(\Bc) \asymp  \sqrt{\frac{ \loglog K}{T}}.$$
The algorithm \CTB described in Section \ref{sec:optimal_alg} attains this rate.
\end{theorem}

As stated in the introduction, the closest literature to our setting is that which concerns sequential estimation of a convex function and noisy convex zeroth order optimisation. Since this literature deals with the continuous case, let us first remark that a straightforward\footnote{By simply discretising the space in $K^{1/\beta}$ bins and applying the method on these bins.} corollary of Theorem~\ref{thm:con_minmax} is that in the case where the arms are in $[0,1]$ and where $f$ is $\beta-$H\"older for some $\beta>0$, the minimax expected regret according to our definition (but in this continuous setting) is $\sqrt{\loglog(T)/T}$, see Appendix~\ref{sec:contin} for details.\\
In~\cite{simchowitz2018adaptive}, the authors present the problem of estimating a convex function by constructing a net of points that is more refined in areas where the function varies more, i.e.~by adapting a quadrature method to the noisy setting. Under an assumption on the modulus of continuity, that is essentially equivalent to assuming that the function is $\beta-$H\"older for some $\beta>0$, the authors provide results in the fixed confidence setting. If one inverses their bounds to go to the fixed budget setting, their results hint toward a lower bound on estimating the convex function in $l_\infty$ norm of order $\sqrt{\log (T)/T}$ and an upper bound of order $\log(T)/\sqrt{T}$. The fact that the logarithmic dependency is much worse in their setting than in ours highlights that the problem of estimating entirely the convex function is more difficult than the problem of estimating a single level set. 
\\
In~\cite{yudin}[Chapter 9], and~\cite{wang2017stochastic, agarwal2011stochastic, liang2014zeroth} the authors consider continuous zeroth order noisy convex optimisation, and focus mainly on reducing the exponent for the dimension $d$ - in this setting the minimax precision for estimating the minimum of the function is conjectured to be $d^{3/2}\mathrm{poly}(\log(T))/\sqrt{T}$ where the $\mathrm{poly}(\log(T))$ term is not really investigated, as the problem is already very difficult as it is. We on the other hand consider mainly $d=1$ and aim at obtaining optimal logarithmic terms.

\vspace{-0.5cm}
\section{Minimax optimal algorithms}\label{sec:optimal_alg}
In this section we present algorithms that match minimax regret rates in Section~\ref{sec:minmaxreg} up to multiplicative constants for \tbp, \stbp, \utbp and \ctbp.

\subsection{Unstructured case \tbp}
\label{sec:alg_unstructured_case}
Given an unstructured problem $\unu \in \B$ we consider the algorithm \unif which samples uniformly across the arms. That is each arm in $[K]$ is sampled $\lfloor T/K\rfloor$ times. The learner then classifies each arm according to its sample mean, see Algorithm~\ref{alg:unif} in Appendix~\ref{app:proof_of_unstr}.

Surprisingly the naive \unif algorithm is optimal in the unstructured case with respect to the lower bound of Theorem~\ref{thm:unstr_minmax}. See the proof of Proposition~\ref{prop:unstr_exp_up} in Appendix~\ref{app:proof_of_unstr}. This contrasts with the related TOP-1 bandit problem where the minimax regret rate is $\sqrt{K/T}$, see~\cite{bubeck2009pure, audibert2009minimax} for hints toward this. This is not very surprising as in the TOP-1 problem we are interested in finding one arm only, namely the arm with highest mean, while in our problem we search for \emph{all arms above threshold} and for this we pay an additional $\sqrt{\log K}$.

\vspace{-0.4cm}
\subsection{Monotone case \stbp}
\label{sec:structured}
 In this section we fix a problem $\unu \in \Bs$. We also assume, in this section, without loss of generality that $\tau \in [\mu_1,\mu_K]$. Indeed, we can always add two deterministic arms $0$ and $K+1$ with respective means $\mu_0 = -\infty$ and $\mu_{K+1} = +\infty.$

We introduce the \stb (Monotone Thresholding Bandits) algorithm. It is composed of two sub-algorithms, \explore and \choose. The first algorithm, \explore, performs a random walk on the set of arms $[K]$ seen as a binary tree, the algorithm \choose then selects, among the visited arms, the one that will be chosen as the threshold for the classification. That is, we choose an arm $\hk$ which leads to the estimator $\hQ$, where $\hQ:\ \hQ[k] = -1 \; \forall k < \hk, \; \hQ[k] = 1 \; \forall k \geq \hk\;.$

\paragraph{Binary Tree}\label{sec:bin}
We associate to each problem $\unu \in \Bs$ a binary tree. Precisely we consider a binary tree with nodes of the form $v=\{L,M,R\}$ where $\{L,M,R\}$ are indexes of arms and we note respectively $v(l)=L, v(r)=R, v(m)=M$. The tree is built recursively as follows: the root is $\root =\{1,\floor*{(1+K)/2}, K\}$, and for a node $v=\{L,M,R\}$ with $L,M,R\in\{1,\ldots,K\}$ the left child of $v$ is $L(v)= \{L,M_l,M\}$ and the right child is $R(v)=\{M,M_r,R\}$ with $M_l = \floor*{(L+M)/2}$ and $M_r = \floor*{(M+R)/2}$ as the middle index between. The leaves of the tree will be the nodes $\{v = \{L,M,R\} : R = L +1\}$. If a node $v$ is a leaf we set $R(v) = L(v) = \emptyset$.  We consider the tree up to maximum depth $H= \floor*{\log_2(K)}+1$. We note $P\big(l(v)\big)=P\big(r(v)\big)$ the parent of the two children and let $|v|$ denote the depth of node $v$ in the tree, with $|\root| = 0$. We adopt the convention $P(\root) = \root$. In order to predict the right classification we want to find the arm whose mean is the one just above the threshold $\tau$. Finding this arm is equivalent to inserting the threshold into the (sorted) list of means, which can be done with a binary search in the aforementioned binary tree. But in our setting we only have access to estimates of the means which can be very unreliable if the mean is close to the threshold. Because of this there is a high chance we will make a mistake on some step of the binary search. For this reason we must allow \explore to backtrack and this is why \explore performs a binary search \emph{with corrections}. Then \choose selects among the visited arms the most promising one.

\paragraph{\explore algorithm} We first define the following integers,

\[T_1 := \lceil 6 \log(K) \rceil \qquad T_2 := \floor*{\frac{T}{3T_1}}\;.\]
The algorithm \explore is then essentially a random walk on said binary tree moving one step per iteration for a total of $T_1$ steps. Let $v_1 = \root$ and for $t < T_1$ let $v_t$ denote the current node, the algorithm samples arms $\{v_t(k) : k \in \{l,m,r\}\}$ each $T_2$ times. Let the sample mean of arm $v_t(k)$ be denoted $\hat{\mu}_{k,t}$. \explore will use these estimates to decide which node to explore next. If an error is detected - i.e. the interval between left and rightmost sample mean do does not contain the threshold, then the algorithm backtracks to the parent of the current node, otherwise \explore acts as the deterministic binary search for inserting the threshold $\tau$ in the sorted list of means. More specifically, if there is an anomaly, $\tau \not \in \left[\hat{\mu}_{l,t},\hat{\mu}_{r,t}\right]$, then the next node is the parent $v_{t+1} = P(v_t)$, otherwise if $\tau \in \left[\hat{\mu}_{l,t},\hat{\mu}_{m,t}\right]$ the the next node is the left child $v_{t+1} = L(v_t)$ and if $\tau \in \left[\hat{\mu}_{m,t},\hat{\mu}_{r,t}\right]$ the next node is the right child  $v_{t+1} = R(v_t)$. If at time $t$, $\tau \in \left[\hat{\mu}_{l,t},\hat{\mu}_{r,t}\right]$ and the node $v_t$ is a leaf then $v_{t+1} = v_t$. See Algorithm \explore for details.

\begin{algorithm}[H]
\caption{\explore}
\label{alg:explore} 
\SetAlgoLined
{\bfseries Initialization:} $v_1=\root$\\
\For{$t=1:T_1$}{
sample $T_2$ times each arm in $v_t$\\
\If{$ \tau \not\in[\hat{\mu}_{l,t},\hat{\mu}_{r,t}]$}{
$v_{t+1} = P(v_t)$ \\
\uElseIf{$R(v_t) = L(v_t) = \emptyset$}{
$v_{t+1} = v_t$}
\uElseIf{$\hat{\mu}_{m,t} \leq \tau \leq \hat{\mu}_{r,t}$}{
$v_{t+1} = R(v_t)$}
\uElseIf{$\hat{\mu}_{l,t} \leq \tau \leq \hat{\mu}_{m,t}$}{
$v_{t+1} = L(v_t)$}
}
}
\end{algorithm}

\paragraph{\choose algorithm} Algorithm \choose takes the history of algorithm \explore, namely the sequence of empirical means $\left(\hat{\mu}_{l,t},\hat{\mu}_{m,t},\hat{\mu}_{r,t}\right)_{t \leq T_1}$ and visited nodes $\left(v_t\right)_{t\leq T_1}$, as the input. In addition it takes as input a parameter $\epsilon > 0$. The action of \choose is to then identify the set of arms among those sampled whose empirical means satisfy one or more of the following: 
\begin{itemize}
    \item their empirical mean is within $\epsilon$ of $\tau$, 
    \item their empirical mean is less than $\tau$ and the empirical mean of the right hand adjacent arm is greater than $\tau$.
\end{itemize}
Here we recognize the set of arms that may lead to a classification with simple regret smaller than $\epsilon$ if the estimates are correct. The algorithm \choose then orders this set by ascending arm index and returns the median, see Algorithm~\ref{alg:choose}.

\begin{algorithm}[H]
\caption{\choose}
\label{alg:choose} 
\SetAlgoLined
{\bfseries Input:} $\epsilon$,  $\left(\hat{\mu}_{l,t},\hat{\mu}_{m,t},\hat{\mu}_{r,t}\right)_{t \leq T_1},  \left(v_t\right)_{t\leq T_1}$\\
{\bfseries Initialization:} $S_1 = [\ ]$\\
\For{$t=1:T_1$}{
 $S_{t+1}=S_t$\\
\If{$\{\exists k \in \{l,m,r\} : |\hmu_{k,t} - \tau| \leq \epsilon\}$
$\lor$ $\{k= v_t (r) = v_t (l) +1;\ \hat{\mu}_{l,t} +\epsilon < \tau \leq \hat{\mu}_{r,t} -\epsilon\}$}{
append $v_t(k)$ to the list $S_{t+1}$ 
}
}
order the list $S_{T_1 + 1}$ by  ascending arm index\\
\Return $\median(S_{T_1+1})$.
\end{algorithm}

\begin{remark}
Note that for any time $t \leq T_1$ we append at most one arm to the list $S_{t+1}$. If at time $t$ there are multiple candidates the choice is made at random. 
\end{remark}

\paragraph{\stb algorithm} The algorithm first runs \explore. We fix a constant $\epsilon_0 = \sqrt{2\log(48)/T_2}\,,$
and compute the parameter $\hepsilon$ with the history of algorithm~\explore,
\begin{equation*}
\hepsilon = \begin{cases} 
2\epsilon_0 &\text{if } \exists (t,k):\  k= v_t (l) = v_t (r) -1;\ \hat{\mu}_{l,t} \leq \tau \leq \hat{\mu}_{r,t} \\
\max\big(2\epsilon_0, \min\limits_{t\leq T_1, k \in\{l,m,r\}} |\hmu_{k,t}-\mu_{v_t(k)}|\big) &\text{else}
\end{cases}
\,.
\end{equation*}

 Then \stb runs the algorithm~\choose with parameter $\hepsilon$. Note that $\hepsilon$ is the smallest parameter greater than $2\epsilon_0$ such that the list $S_{T_1+1}$ is non empty. This choice will become clear in the proof of Theorem~\ref{thm:mon_minmax} in Appendix~\ref{app:mon_minimax}. Morally it allows to select a majority of ``good" arms (i.e that provide a low regret classification $\hQ$) in $S_{T_1+1}$ such that the median $\hk$ is also a ``good" arm, see Algorithm~\ref{alg:STB2}.

\begin{algorithm}[H]
\caption{\stb}
\label{alg:STB2} 
\SetAlgoLined

{\bfseries run} algorithm \explore
\begin{itemize}

    \item Output: $\left(\hat{\mu}_{l,t},\hat{\mu}_{m,t},\hat{\mu}_{r,t}\right)_{t \leq T_1}$,  $\left(v_t\right)_{t\leq T_1}$
\end{itemize}

{\bfseries run} algorithm \choose
\begin{itemize}
    \item Input: $\hepsilon,\, \left(\hat{\mu}_{l,t},\hat{\mu}_{m,t},\hat{\mu}_{r,t}\right)_{t \leq T_1}$,  $\left(v_t\right)_{t\leq T_1}$
    \item Output: arm index $\hk$ 
\end{itemize}

\Return  $(\hk, \hQ) :\quad \hQ_k=  2\ind_{\{k \geq \hk \}}-1$
\end{algorithm}
\vspace{-0.2cm}
The \stb algorithm will achieve the minimax rate on expected simple regret given in Theorem \ref{thm:mon_minmax}, see the proof of Theorem \ref{thm:mon_minmax}, in Appendix~\ref{app:mon_minimax}, for details. 

\vspace{-0.2cm} 
\begin{remark}[Adaptation of \stb to a non-increasing sequence, \dstb]\label{rem:DSTB}
\stb is applied for a monotone non-decreasing sequence $(\mu_k)_k$, and it is easy to adapt it to a monotone non-increasing sequence $(\mu_k)_k$. In this case, we transform the label of arm $i$ into $K-i$, and apply \stb to the newly labeled problem - where the mean sequence in now non-decreasing. We refer to this modification as \dstb.

\end{remark}

\vspace{-0.1cm}

\vspace{-0.6cm}
\subsection{Unimodal case \utbp}

We now turn to the algorithm for the unimodal case, \UTB (Unimodal Thresholding Bandits) algorithm. This algorithm is based on the algorithm $\stb$, and on any black-box algorithm that is minimax-optimal for TOP-1 simple regret on $\cB$, as described in~\cite{bubeck2009pure}. We name such an algorithm \SR; it takes no parameter and returns an arm $\hat m$. Since \SR is minimax optimal for the TOP-1 simple regret, we have on any problem $\nu \in \cB$ with means $(\mu_k)_k$ and maximal mean $\mu^*$, that if $\SR$ is run for $T$ times, then
$$\mathbb E_\nu [\mu^* - \mu_{\hat m}] \leq c_{\SR} \sqrt{\frac{K}{T}},$$
where $c_{\SR}>0$ is a universal constant. Note that taking MOSS from~\cite{audibert2009minimax} and modifying it so that it outputs $\hat m$ as being sampled at random according to the proportion of times that each arm was sampled by MOSS, is minimax-optimal algorithm for the TOP-1 problem.

The idea of \UTB is to start by running \SR on a fraction of the budget, and take its output $\hat m$. Then we run respectively \stb on $\{1, \ldots, \hat m\}$, and \dstb on $\{\hat m, \ldots, K\}$ on a fraction of the budget. They respectively return $\hat l, \hat r$. We then  use the last fraction of the budget to sample all arms in $\{\hat l, \hat r, \hat m, \hat l-1, \hat r+1\}$ and compute the respective empirical means $\hmu_{k}$ for $k$ being one of these arms. If $\hat l, \hat r$ seem either close enough to the threshold, or seem above while the adjacent arm seems below, we predict $\{\hat l, \ldots, \hat r\}$ as the set of arms above threshold. Otherwise we return the empty set, see Algorithm~\ref{alg:UTB}.

This intuitively makes sense as $\hat m$ is an estimator of the maximum $k^*$ of the mean sequence and unimodality implies that $(\mu_k)_{k \leq k^*}$  is non-decreasing, and that $(\mu_k)_{k \geq k^*}$  is non-increasing. So $\hat l, \hat r$ are estimators of the points where the mean sequence crosses the threshold, respectively on the left and on the right of the estimator of the maximum. The last step - where we compute empirical means and check based on them if the outputs seem reasonable - is a checking step for making sure that the output of \SR is not so close to threshold (or flawed), that the outputs of \stb and \dstb are completely flawed.

\begin{algorithm}[H]
\caption{\UTB}
\label{alg:UTB} 
\SetAlgoLined
{\bfseries Initialization:} $\hat m$ =  output of $\SR$ with budget $\lfloor T/4 \rfloor$,\\
$\hat l =\hat k$ output  of \stb with arms $\{1,\ldots, \hat m\}$, threshold $\tau$, budget $\lfloor T/8 \rfloor$, \\
 $\hat r = \hat k$ output of \dstb with arms $\{ \hat m, \ldots, K\}$, threshold $\tau$, budget $\lfloor T/8 \rfloor$, \\
Sample $\hat m, \hat l,\hat r, \hat l-1, \hat r+1$ each $\lfloor T/10 \rfloor$ times\\
\If{$\small \big(\{\hmu_{\hat l-1}  < \tau < \hmu_{\hat l} \} \lor \{|\hmu_{\hat l}-\tau| \leq  \hmu_{\hat m} - \tau\}\big) \land \big(\{\hmu_{\hat r} < \tau < \hmu_{\hat r+1} \}\big) \lor \{|\hmu_{\hat r}-\tau| \leq  \hmu_{\hat m} - \tau\}\big)$}{
 $\hat S= \{\hat l, \ldots, \hat r\}$\\
\Else{ $\hat S = \emptyset$}}

\Return  $\hQ :\quad \hQ_k=  2\ind_{\{k \in \hat{S}\}}-1$
\end{algorithm}

\vspace{-0.8cm}
\subsection{Concave case \ctbp}

In this section, we present the \CTB algorithm, which is based on several applications of \stb. We first define the following \textit{log-sets}. Consider two integers $ l \leq r $ and the associated set $\{l, l+1, \ldots, r\}$. We write $\mathcal S^{\log}_{l,r} = \{l, l+1, l+2, l+2^2, \ldots, (l+2^a) \land \lfloor (l+r)/2\rfloor\}$, where $a$ is the smallest integer such that $l+2^a \leq r \leq l+2^{a+1}$.

 Algorithm \CTB proceeds in phases. At phase $i$ an interval $\{l_i, \ldots, r_i\}$ is refined from both ends by applying \stb and \dstb. Algorithm \CTB makes sure that with high probability, the regret of $\{l_i, \ldots, r_{i}\}$, is bounded by $\epsilon_i = (7/8)^i$. A very important idea of \CTB is that it does not apply \stb and \dstb on $\{l_i, \ldots, r_{i}\}$ but thanks to the \emph{concavity} only on the \textit{log-sets associated to $\{l_i, \ldots, r_{i}\}$}. I.e.~we will apply $\stb$ on $\mathcal S^{\log}_{l_i,r_i}$ and $\dstb$ on $-\mathcal S^{\log}_{-r_i,-l_i}$. This allows us to have much shorter phases as the two log-sets contain about $\log(r_i - l_i)$ arms, instead of $r_i - l_i$ arms.

We now describe formally \CTB. The algorithm \CTB consists of two  sub-routines, an iterative application of \stb and then a decision rule based on the collected samples. These routines are respectively the \textbf{for loop} and \textbf{if statement} in the \ctb algorithm. 
\paragraph{Iterative application of \stb.} For $\tilde M>0$ and $i<M$ we set
\begin{align*}
\delta_i^{(\tilde M)} = 2^{i-\tilde M} \qquad \epsilon_i = \left(1-\frac{1}{8}\right)^i  \qquad \tau_i = \tau - \frac{3}{4}\epsilon_i, \qquad
T_2^{(i)}(\tilde M) =  \left\lfloor \frac{2^{14} \log\log K}{\epsilon_i^2}\log\left(\frac{1}{\delta_i^2}\right)\right\rfloor,
\end{align*}
and let $M$ be the largest integer such that
$6\sum_{i \leq \tilde M} T_2^{(i)}(\tilde M) \leq T$.
In what follows we write
$$\delta_i:=\delta_i^{(M)},~~~~T_2^{(i)} = T_2^{(i)}(M).$$

\ctb proceeds in $M$ phases and at each it updates a set of three arms $l_i \leq m_i \leq r_i$ - where $m_i$ is at the middle between $l_i$ and $r_i$. It first samples all these arms - as well as $l_i - 1, r_i+1$ - $T_2^{(i)}$ times, and these samples are used to compute empirical means $\hmu_{p,i}$ for $p \in\{m, l, r,l-1,r+1\}$ - corresponding respectively to the arms $\{m_i, l_i, r_i,l_i-1,r_i+1\}$. It then runs respectively $\stb$ on $\mathcal S^{\log}_{l_i,r_i}$ and $\dstb$ on $-\mathcal S^{\log}_{-r_i,-l_i}$, both with threshold $\tau_i$ and budget $T_2^{(i)}$. These routines output $l_{i+1}, r_{i+1}$, and we define $m_{i+1}$ as the middle between these arms. 

\vspace{-0.2cm}
\paragraph{Decision rule} The second sub routine of \ctb is a decision rule between all $l_{i}, r_{i}$, for finding the right scale, based on the arms and empirical means collected in the previous routine. It takes the $l_{i}, r_{i}$ that are as close as possible to arms $m_i$ far from threshold, but that are close to threshold - and it outputs a set $\hat S$. Finally \CTB classifies this set as being above threshold. Set
\begin{align*}
    &\mathcal I_m = \{m_i : \hmu_{m,i} \geq \tau + 2\epsilon_i\},~~\mathrm{and}\\
\mathcal I_l = \{l_i : \hmu_{l,i}\geq \tau- 2\epsilon_i,~~ &\hmu_{l-1,i}\leq \tau  -  \frac{\epsilon_i}{4}\},\mathrm{and}~~\mathcal I_r = \{r_i : \hmu_{r,i}\geq \tau- 2\epsilon_i, \hmu_{r+1,i}\leq \tau -  \frac{\epsilon_i}{4}\}.
\end{align*}

\begin{algorithm}[H]
\SetAlgoLined
\caption{\CTB}
\label{alg:CTB} 
{\bfseries Initialization:}  $l_0 = 1, r_0 = K, m_0 = \lfloor\frac{l_0 + r_0}{2} \rfloor$\\ 
\For{$i = 1:M$}{

sample arms $l_i, l_i - 1, r_i, r_i+1$ and $m_i$ each $T_2^{(i)}$ times.\\
$l_{i+1} =$ output $\hat k$ of \stb with arms $\mathcal S^{\log}_{l_i,r_i}$, threshold $\tau_i$, budget $T_2^{(i)}$ \\

$r_{i+1} =$ output $\hat k$ of \dstb with arms $-S^{\log}_{-r_i, -l_i}$, threshold $\tau_i$, budget $T_2^{(i)}$ \\
$m_{i+1} = \lfloor\frac{l_{i+1} +r_{i+1}}{2} \rfloor$
}

\If{$\mathcal I_m = \emptyset$}{
Set $\hat S = \emptyset$\\
\Else{
Set $\hat l = \max \{k \in \mathcal I_l, k \leq \min_i \mathcal I_m\}$\\
Set $\hat r = \min \{k \in \mathcal I_r, k \geq \max_i \mathcal I_m\}$\\
Set $\hat S = \{\hat l, \ldots, \hat r\}$\\
}
}

\Return  $\hQ :\quad \hQ_k=  2\ind_{\{k \in \hat{S}\}}-1$

\end{algorithm}

\paragraph{Acknowledgements.} The work of J.~Cheshire is supported by the Deutsche Forschungsgemeinschaft (DFG) DFG - 314838170, GRK 2297 MathCoRe.
The work of P.~M\'enard  is  supported by the European CHISTERA project DELTA, and partially supported by the SFI Sachsen-Anhalt for the project RE-BCI, and by the UFA-DFH through the French-German Doktorandenkolleg CDFA 01-18. The work of A. Carpentier is partially supported by the Deutsche Forschungsgemeinschaft (DFG) Emmy Noether grant MuSyAD (CA 1488/1-1), by the DFG - 314838170, GRK 2297 MathCoRe, by the DFG GRK 2433 DAEDALUS (384950143/GRK2433), by the DFG CRC 1294 'Data Assimilation', Project A03, and by the UFA-DFH through the French-German Doktorandenkolleg CDFA 01-18 and by the UFA-DFH through the French-German Doktorandenkolleg CDFA 01-18 and by the SFI Sachsen-Anhalt for the project RE-BCI.

\newpage

\bibliography{biblio2}
\newpage
\tableofcontents
\newpage

\appendix

\clearpage
\section{Adaptation to the \texorpdfstring{$\beta$}{TEXT}-H\"older continuous case}\label{sec:contin}
In this section we explain how our results can be adapted in a very simple way to the case where the arms are not $\{1,\ldots,K\}$ but the continuous set $[0,1]$, and where the mean sequence $(\mu_k)_{k\in [0,1]}$ is now a function. We assume, on top of the fact that the distributions are supported in $[0,1]$, that the mean function $\mu$ is $\beta$-H\"older for some constant $\beta>0$, i.e.~in the case $\beta \leq 1$ and a constant $L>0$ such that $\forall x,y \in [0,1]$, $|\mu_x - \mu_y| \leq L |x - y|^\beta$. In this case, straightforward corollaries of our results imply the minimax regret rates in Table~\ref{tab:KV1}.

In oder to get these results, it is sufficient to divide $[0,1]$ in $M$ intervals of same size and adapt the results as usually done in the non-parametric litterature (by controlling the bias). We need to choose (i) $M$ as $\Big(\frac{T}{\log T}\Big)^{\frac{1}{2\beta+1}}$ in \tbp, (ii) $M$ as $T^{1/\beta}$ in \stbp, (iii) $M$ as $T^{\frac{1}{2\beta+1}}$ in \utbp, and (iii) $M$ as $T^{1/\beta}$ in \ctbp.

Interestingly, the rates of \stbp and \ctbp \textit{do not depend on $\beta$} - but note that $\beta$ plays a role in the multiplicative constants in front of the rate, i.e.~the smaller $\beta$, the larger the constant. On the other hand the rates in \tbp and \utbp depend on $\beta$. Note that this is a phenomenon \emph{specific to the $1$-dimensional case}. Indeed, finding the level set of a monotone and of a convex function in dimension $d$ is typically done at a much slower rate, depending on $\beta$ and $\d$.

\begin{table}[!h]
\begin{center}
\renewcommand{\arraystretch}{1.7}
 \begin{tabular}{c||c|c|c|c}
  Our results & Unstructured   & Monotone  & Unimodal & Convex   \\
    &  \tbp  &  \stbp &  \utbp &  \ctbp  \\\hline \hline
 K-arms& $\sqrt{\frac{K\log K}{T}}$ & $\sqrt{\frac{\log K \lor 1}{T}}$ & $\sqrt{\frac{K}{T}}$ & $\sqrt{\frac{\log\log K \lor 1}{T}}$ \\ \hline
 $\beta$-H\"older& $\Big(\frac{\log T}{T}\Big)^{\frac{\beta}{2\beta+1}}$ & $\sqrt{\frac{\log T \lor 1}{T}}$ &  $\Big(\frac{1}{T}\Big)^{\frac{\beta}{2\beta+1}}$  & $\sqrt{\frac{\log\log T \lor 1}{T}}$
 \end{tabular}

 \caption{Order of the minimax expected regret for the thresholding bandit problem, in the case of all four structural assumptions on the means of the arms considered in this paper. All results are given up to universal multiplicative constants. The first line concerns the $K-$armed setting of the main paper, and the second line concerns the $\mathcal X$-armed setting where the set of arms is $[0,1]$ and where the function is $\beta$-H\"older (on top of the shape constraints).}
 \label{tab:KV1}
\end{center}
 \end{table}

\section{Extension to \texorpdfstring{$\sigma^2$}{TEXT}-sub-Gaussian for \tbp and \stbp} \label{sunGaus}
While in the main text for simplicity we only consider distributions bounded on the $[0,1]$ interval all proofs relating to the \tbp and \stbp given in the appendix will extend to the sub Gaussian case. The lower bound for the \ctbp will also extend to the sub Gaussian case. That is we redefine the setting as follows: the learner is presented with a $K$-armed bandit problem $\unu =\{\nu_1,\ldots,\nu_K\}$, where $\nu_k$ is the unknown distribution of arm $k$. Let $\sigma^2 >0$, all arms are assumed to be $\sigma^2$-sub-Gaussian as described in the following definition, we write $\mu_k$ for the mean of arm $k$.

\begin{definition}[$\sigma^2$-sub-Gaussian]\label{def:subgau}
A distribution $\nu$ of mean $\mu$ is said to be $\sigma^2$-sub-Gaussian if for all $t \in \R$ we have,
$$\mathbb{E}_{X \sim \nu}\big[e^{t\left(X - \mu\right)}\big] \leq \mathrm{exp}\left(\frac{\sigma^2 t^2}{2}\right)\,.$$
\end{definition}
In particular the Gaussian distributions with variance smaller than $\sigma^2$ and the distributions with absolute values bounded by $\sigma$ are $\sigma^2$-sub-Gaussian.

The only adaptation that has to be made to accomodate this case in the \stb algorithm is to define
$$\epsilon_0 =\sqrt{\frac{2\sigma^2\log(48)}{T_2}}.$$

\section{Proof of Theorem~\ref{thm:unstr_minmax}}\label{app:proof_of_unstr}
In the proof of all results in this section, we assume that the more general sub-Gaussian assumption described in Section~\ref{sunGaus} is satisfied - and not necessarily that the distributions of all arms are bounded on the $[0,1]$ interval. We explain in the proof how the lower bound can be straightforwardly adapted to distributions supported in $[0,1]$.

 We denote the Kullback-Leibler divergence between two Bernoulli distributions $\ber(p)$ and $\ber(q)$ (with the usual conventions) by 
\[
\kl(p,q)=p \log\frac{p}{q} +(1-p) \log \frac{1-p}{1-q}\,.
\]

\begin{algorithm}[H]
\caption{\unif}\label{alg:unif} 
\SetAlgoLined
\For{$k=1:K$}{
    Sample arm $k$ a total of $\lfloor\frac{T}{K}\rfloor$ times.\\
    Compute $\hmu_k$ the sample mean of arm $k$. 
}
\Return\vspace{-5mm} \begin{equation*}\hQ :\quad    \hQ_k=
    \begin{cases}
      -1 & \text{if}\ \hmu_k < \tau \\
      1 & \text{if}\ \hmu_k \geq \tau
    \end{cases}
  \end{equation*}

\end{algorithm}
During this section we will prove Theorem \ref{thm:unstr_minmax} by first demonstrating a lower bound on expected regret across $\cB$ and then showing that the \unif algorithm achieves said lower bound. We first prove the following proposition to establish a lower bound.
\begin{proposition}\label{prop:unstr_lo}
For any $T\geq 1$ and any strategy $\pi$, there exists a unstructured bandit problem $\unu \in \B$, such that
\[
\bR^{\pi,\unu}_T \geq \frac{3}{4}\sqrt{\frac{\sigma^2 \max\big(2,\log(K)\big) K }{8T}}\,.
\]
\end{proposition}

\begin{proof} 
Without loss of generality we can assume that $\tau = 0$. Fix some positive real number $0 < \epsilon < 1$. And consider the family of Gaussian bandit problems indexed by an vertex of the unite hyper-cube of dimension $K$, id est $Q\in \{-1,1\}^K$
\[ 
\unu^Q = \big(\Ng(Q_1 \epsilon,\sigma^2),\ldots,\Ng(Q_K \epsilon,\sigma^2)\big)\,,
\]
and note that if we wish to consider distributions supported in $[0,1]$ we can consider instead $\tau =1/2$ and
\[ 
\unu^Q = \big(\mathcal B(1/2+Q_1 \epsilon),\ldots,\mathcal B(1/2+Q_K \epsilon)\big)\,,
\]
up to minor adaptations of the constants, and to considering $\tau = 1/2$. Note that all these bandit problems belong to the set of unstructured bandit problems, $\unu^Q\in\B$. 

The regret in the bandit problem $\unu^Q$ of the strategy $\pi$ can be rewritten as follows
\begin{align*}
\bR_T^{\unu^Q,\pi} &= \epsilon \EE_{Q} \max_k \ind_{\{ \hQ_k \neq Q_k\}}\\
&= \epsilon (1-\EE_{Q}\ind_{\{\hQ=Q\}})\,,
\end{align*}
where we denote by $\EE_Q$ the expectation under the bandit problem $\unu^Q$. We will provide a minimax lower bound on the regret by using the classic Fano inequality. We first lower bound the minimax expected regret in the problem $\unu^Q$ by the Bayesian regret with a uniform distribution over the bandit problems $\unu^Q$,
\begin{align}
\max_{Q}\bR_T^{\unu^Q,\pi} \geq \epsilon \left( 1- \frac{1}{2^K}\sum_Q \EE_Q \ind_{\{\hQ=Q\}}\right)\,. \label{eq:minimax_to_bayesian}
\end{align}
Let $Q^k$ be the transformation of $Q$ that flip the sign of the coordinate $k$,
\[
Q_a^k = \begin{cases}
Q_a \text{ If }a\neq k,\\
-Q_a \text{ If }a = k\,.
\end{cases}\]
Thanks to the contraction and the convexity of the relative entropy, see \citet{gerchinovitz2017fano}, we have 
\begin{align*}
\kl\!\left(\frac{1}{K}\sum_{k=1}^K \EE_{Q^k} \ind_{\{\hQ=Q^k\}},\underbrace{\frac{1}{K}\sum_{k=1}^K \EE_Q \ind_{\{\hQ=Q^k\}}}_{\leq 1/K}\right)\leq \frac{1}{K}\sum_{k=1}^K \EE_{Q^k}\! \big[N_k(T)\big] \frac{\epsilon^2}{2 \sigma^2}\,,
\end{align*}
where $N_k(T)= \sum_{t=1}^T \ind_{\{k_t=k\}}$ denotes the number of times in total arm $k$ is sampled. Then using a refined Pinsker inequality (see \citet{gerchinovitz2017fano}) $\kl(x,y)\geq (x-y)^2 \max\big(2,\log(1/y)\big)$, we obtain
\begin{equation}
\frac{1}{K}\sum_{k=1}^K \EE_{Q^k} \ind_{\{\hQ=Q^k\}} \leq \frac{1}{K} + \sqrt{\frac{1}{K}\sum_{k=1}^K \EE_{Q^k}\! \big[N_k(T)\big] \frac{\epsilon^2}{2\sigma^2 \max\big(2,\log(K)\big)}}\,.
\label{eq:lb_unstructured_fano}
\end{equation}
Therefore thanks to the concavity of the square root, we can average over all the bandit problems $\unu^Q$
\[
\frac{1}{2^K}\sum_Q \frac{1}{K}\sum_{k=1}^K \EE_{Q^k} \ind_{\{\hQ=Q^k\}} \leq \frac{1}{K} + \sqrt{\frac{1}{2^K}\sum_Q\frac{1}{K}\sum_{k=1}^K \EE_{Q^k}\! \big[N_k(T)\big] \frac{\epsilon^2}{2\sigma^2\max\big(2,\log(K)\big)}}\,.
\]
Now it remains to remark that by symmetry 
\begin{align*}
\sum_Q \sum_{k=1}^K \EE_{Q^k} \ind_{\{\hQ=Q^k\}} &= \sum_{Q'} \sum_{k=1}^K \EE_{Q'} \ind_{\{\hQ=Q'\}} = K \sum_Q \EE_Q \ind_{\{\hQ=Q\}}\,,\\
\sum_Q\sum_{k=1}^K \EE_{Q^k}\! \big[N_k(T)\big] &= \sum_{Q'}\sum_{k=1}^K \EE_{Q'}\! \big[N_k(T)\big] = \sum_Q T\,.
\end{align*}
Hence from \eqref{eq:lb_unstructured_fano} we get 
\[
\frac{1}{2^K}\sum_Q \EE_{Q} \ind_{\{\hQ=Q\}} \leq \frac{1}{K} + \sqrt{ \frac{T \epsilon^2}{2 K \sigma^2 \max\big(2,\log(K)\big)}}\,,
\]
and then from~\eqref{eq:minimax_to_bayesian} we obtain 
\[
\max_{Q}\bR_T^{\unu^Q,\pi} \geq \epsilon \left( \frac{1}{2}- \sqrt{ \frac{T \epsilon^2}{2 K \sigma^2 \max\big(2,\log(K)\big)}}\right)\,.
\]
Choosing $\epsilon = \sqrt{K \sigma^2 \max\big(2,\log(K)\big)/(8 T)}$ allows us to conclude.
\end{proof}
We next prove the following proposition to establish a upper bound on the regret of the \unif algorithm with high probability, 

\begin{proposition}\label{prop:unstr_hiprob_up}
For any unstructured bandit problem $\unu \in \cB$, any $T\geq K$, any $0 < \delta < 1$, \unif satisfies \vspace{-1mm}
$$\mathbb P_{\unu}\left(R_T^{\unif, \unu} \geq \sqrt{\frac{ 4 \sigma^2 K}{T}\log\left(\frac{2 K}{ \delta}\right)}\right) \leq \delta\,.$$
\end{proposition} 

\begin{proof} 

During the execution of the \unif algorithm $\forall k \in \{1,...,K\}$ arm $k$ is sampled $\lfloor T/K \rfloor$ times with sample mean $\hat{\mu}_k$. Let $\delta > 0$ and consider the event,

\[\xi := \left\{ \forall \; k \leq K,\; \left|\hat{\mu}_k - \mu_k \right| \leq \sqrt{\frac{ 4 \sigma^2 K}{T}\log\left(\frac{2 K}{\delta}\right)}\right\} \;.\]

Thanks to the Hoeffding inequality and an union bound this event occurs with probability greater than $1-\delta$. As under the event $\xi$, 

$$\hmu_k \in \left[\hmu_k - \sqrt{\frac{ 4 \sigma^2 K}{T}\log\left(\frac{2 K }{\delta}\right)}, \hmu_k + \sqrt{\frac{ 4 \sigma^2 K}{T}\log\left(\frac{2 K}{ \delta}\right)}\right]\;,$$ 

and the returning classification is

\begin{equation*}\hQ :\quad    \hQ_k=
    \begin{cases}
      -1 & \text{if}\ \hmu_k < \tau \\
      1 & \text{if}\ \hmu_k \geq \tau
    \end{cases}\,,
  \end{equation*}

we have with probability at least $1-\delta$ 
$$R_T = \max_{\{k\in[K]:\ \hQ_k \neq Q_k\}} \Delta_k
\leq  \sqrt{\frac{ 4 \sigma^2 K}{T}\log\left(\frac{2 K}{ \delta}\right)} \;.$$
\end{proof}

We are now able to demonstrate a bound on the expected regret of the \unif algorithm. 
\begin{proposition}\label{prop:unstr_exp_up}
For any unstructured bandit problem $\unu \in \cB$, and any $T\geq K$, \unif satisfies \vspace{-1mm}
$$\bR_T^{\unif, \unu} \leq   7\sqrt{\frac{\sigma^2\log(2K) K}{T}}\;.\vspace{-3mm}$$
\end{proposition}
\begin{proof} 
By application of Theorem \ref{prop:unstr_hiprob_up}, for $\epsilon > 0$ we have, 

$$\PP\left(R_T \geq \epsilon\right) \leq 2 K\exp\left(-\epsilon^2 \frac{T}{4 \sigma^2 K}\right).$$
Hence for $\epsilon_0 = \sqrt{4\sigma^2\log(2K) K/T}$ integrating these probabilities we obtain an upper bound on the expected simple regret
\begin{align*}
    \bR_T  &\leq \sqrt{2}\epsilon_0 + \int_{\sqrt{2}\epsilon_0}^{+\infty} \exp\left(-(\epsilon^2-\epsilon_0^2) \frac{T}{2 \sigma^2 K}\right) \dIff\epsilon\\
    &\leq \sqrt{2}\epsilon_0 + \int_{0}^{+\infty} \exp\left(-\epsilon^2 \frac{T}{8 \sigma^2 K}\right) \dIff\epsilon\\
    & =  \sqrt{\frac{8\sigma^2\log(2K) K}{T}} + \sqrt{\frac{2\pi\sigma^2 K}{T}}\\
    &\leq 7\sqrt{\frac{\sigma^2\log(2K) K}{T}}\,.
\end{align*}

\end{proof}
Setting $\sigma = 1$, Theorem \ref{thm:unstr_minmax} follows directly from a combination of Propositions \ref{prop:unstr_exp_up} and \ref{prop:unstr_lo}. 

\section{Proof of Theorem \ref{thm:mon_minmax}}\label{app:mon_minimax}
In the proofs of all results in this section, we assume that the more general sub-Gaussian assumption described in Section~\ref{sunGaus} is satisfied - and not necessarily that the distributions of all arms are bounded on the $[0,1]$ interval. In this case, we remind that we redefine $\epsilon_0$ as in Section~\ref{sunGaus}. Also, we explain in the proof of the lower bound how it is possible to straightforwardly adapt the proof to the case where the distributions are supported in $[0,1]$.

During this section we will prove Theorem \ref{thm:mon_minmax} by first demonstrating a lower bound upon expected regret in the \stbp setting, Proposition \ref{prop:mon_lo}. We will then go on to provide an upper bound on the regret of the \stb with high probability, Proposition \ref{prop:mo_hiprob_up} which will be used to finally prove Corollary \ref{cor:mo_exp_up} which provides a optimal bound for the \stb in expected regret. Setting $\sigma = 1$ Theorem \ref{thm:mon_minmax} will then follow directly from Proposition \ref{prop:mon_lo} and Corollary \ref{cor:mo_exp_up}.

\begin{proposition}\label{prop:mon_lo}
For any $T\geq 1$ and any strategy $\pi$, there exists a structured bandit problem $\unu \in \Bs$, such that
\[
\bR^{\pi,\unu}_T \geq \frac{1}{8}\sqrt{\frac{\sigma^2 \max\big(2,\log(K)\big)}{8T}}\,.
\]

\end{proposition} 
\begin{proof}
We will proceed as in the proof of Proposition~\ref{prop:unstr_lo}. Fix some positive real number $0 < \epsilon < 1$. Without loss of generality we can assume that $\tau = \epsilon/2$. And consider the family of Gaussian bandit problems $\unu^k$ indexed by $k\in\{0,\ldots,K\}$, such that for all $k \in\{0,\ldots,K\}$, $l \in [K]$,
\[
\nu_l^k=
    \begin{cases}
      \Ng(0, \sigma^2) & \text{if}\ k < l \\
      \Ng(\epsilon, \sigma^2) & \text{else}
    \end{cases}\,.
\]
Note that if we wish to consider distributions supported in $[0,1]$ we can consider instead $\tau=1/2+\epsilon/2$ and
\[
\nu_l^k=
    \begin{cases}
      \mathcal B(1/2) & \text{if}\ k < l \\
      \mathcal B(1/2+\epsilon) & \text{else}
    \end{cases}\,.
\]
up to minor adaptations of the constants, and to considering $\tau = 1/2$. 

Note that all these bandit problems belong to the set of structured bandit problems, $\unu^k\in\B$. Following the same steps as in the proof of Proposition~\ref{prop:unstr_lo} one can lower bound the maximum of the expected regrets over all the bandit problems introduced above,
\[
\max_{k\in[K]}\bR_T^{\unu^k,\pi} \geq \frac{\epsilon}{2} \left( 1- \frac{1}{K}\sum_{k=1}^K \EE_k \ind_{\{\hQ=Q^k\}}\right)\,, 
\]
where we denote by $\EE^k$ the expectation and by $Q^k$ the true classification in the problem $\unu^k$. Thanks to the contraction and the convexity of the relative entropy we have 
\begin{align*}
\kl\!\left(\frac{1}{K}\sum_{k=1}^K \EE_{k} \ind_{\{\hQ=Q^k\}},\underbrace{\frac{1}{K}\sum_{k=1}^K \EE_0 \ind_{\{\hQ=Q^k\}}}_{\leq 1/K}\right) &\leq \frac{1}{K}\sum_{k=1}^K  \sum_{l=k}^K\EE_{k}\!\big[N_l(T)\big] \frac{\epsilon^2}{2 \sigma^2}\\
&\leq \frac{T \epsilon^2}{2 \sigma^2}\,.
\end{align*}
Then using a refined Pinsker inequality $\kl(x,y)\geq (x-y)^2 \max\big(2,\log(1/y)\big)$, we obtain
\begin{equation*}
\frac{1}{K}\sum_{k=1}^K \EE_{k} \ind_{\{\hQ=Q^k\}} \leq \frac{1}{K} + \sqrt{\frac{ T\epsilon^2}{2\sigma^2 \max\big(2,\log(K)\big)}}\,.
\end{equation*}
Hence combining the last three inequalities we get 
\[
\max_{k\in[K]}\bR_T^{\unu^k,\pi} \geq \frac{\epsilon}{2} \left( \frac{1}{2}- \sqrt{ \frac{T \epsilon^2}{2\sigma^2 \max\big(2,\log(K)\big)}}\right)\,.
\]
Choosing $\epsilon = \sqrt{\sigma^2 \max\big(2,\log(K)\big)/(8 T)}$ allows us to conclude.
\end{proof}

We next prove the following to proposition to establish an upper bound on the simple regret of the \stb algorithm with high probability and then prove Corollary \ref{cor:mo_exp_up} to establish an upper bound on the expected regret of the \stb algorithm. For Proposition \ref{prop:mo_hiprob_up} we consider a more general set of problems, given $\epsilon > 0$, define, 

$$\Bs^{*,\epsilon} := \{\cB : \left( \min(|\mu_i - \tau|,\epsilon)\sign(\mu_i - \tau)\right)_{k \leq K} \; \mathrm{is\;an\; increasing\;sequence} \}\;. $$

Note that for all $\epsilon > 0$, $\Bs \subset \Bs^{*,\epsilon}$, hence all results will hold also in the unaltered monotone setting.   

\begin{proposition}\label{prop:mo_hiprob_up}
For any $\epsilon > \epsilon_0$ and any problem $\unu \in \Bs^{*,\epsilon}$, and any $T > 6 \log(K)$, the \stb Algorithm will achieve the following bound on simple regret,

\begin{equation*}
    \PP_{\unu}(R_T^{\stb, \unu} \geq \epsilon ) \leq \min\left( \exp\!\Bigg(-\frac{3\log(K)}{4}\right) ,\;
72 \log(K) \exp\!\left(-\frac{T\epsilon^2}{216 \sigma^2 \log(K)}\right) \Bigg)\,. 
\end{equation*}

\end{proposition}

\begin{corollary}\label{cor:mo_exp_up}
For any problem $\unu \in \Bs$ and any $T \geq 12\log(K)$, the \stb algorithm will achieve the following bound on expected regret, 

$$\bR^{\stb, \unu}_T \leq 80 \sqrt{\frac{\sigma^2 \log(K)}{T}}\,.$$

\end{corollary}

The proof of Proposition \ref{prop:mo_hiprob_up} and Corollary \ref{cor:mo_exp_up} is structured in several steps which we will first summarise. For a level $\epsilon > 0$ we define a set of ``good nodes" containing ``$\epsilon$-good arms", those which when outputted will achieve the bound $R_T < 2\epsilon$. In Proposition \ref{prop:Z} we prove these nodes form a "consecutive tree", see Definition \ref{def:consecutive}. At time $t$ we say we have a ``favourable event" if all sampled empirical means are within $\epsilon$ of the true mean, In this case we say the algorithm makes a ``good decision'', see \eqref{eq:gdef}. In Lemma \ref{lem:dxic} we prove that on every good decision we move towards the set of good arms or remain within them. Lemma \ref{lem:lb_on_Gt} then shows that provided we make enough good decisions the number of good arms in $S$ is large. We can then bound the probability of making a high proportion of good decisions, see Lemma \ref{lem:fisrt_high_prb_regret_bound}, to give an upper bound on regret. This in combination with a second upper bound, Lemma \ref{lem:second_high_prb_regret_bound}, will give our result. 

\paragraph{Step 0: Definitions and Lemmas}

We will use the following definitions.
\begin{definition}\label{def:subtree}
We define the subtree $ST(v)$ of a node $v$ recursively as follows: $v \in ST(v)$ and
$$\forall \; q \in ST(v), \; L(q), R(q) \in ST(v)\;.$$
\end{definition} 

\begin{definition}\label{def:consecutive} 
A consecutive tree $U$ with root $u_\root$ is a set of nodes such that $u_{\root} \in U$ and 

\[\forall v \in U : v \neq u_{\root},\, P(v) \in U. \]
with the additional condition, 
\[\root \in U \Rightarrow u_{\root} = \root\]
where $\root$ is the root of the entire binary tree.

\end{definition}

We define $Z^\epsilon$, the set of $\epsilon$-good nodes, as the union of the two sets
\begin{equation}\label{eq:z1}
Z_1^\epsilon := \{ v : \exists k \in \{l,m,r\} : |\mu_{v(k)} - \tau| \leq \epsilon\}\;,
\end{equation} 
\begin{equation}\label{eq:z2} 
Z_2^\epsilon := \{ v : v(r) = v(l) + 1;\ \mu_{v(l)} \leq \tau \leq \mu_{v(r)} \}\backslash Z_1^\epsilon \;,
\end{equation}
that is

$$Z^\epsilon := Z_1^\epsilon \; \cup \; Z_2^\epsilon \;.$$ 

It is important to note that

\begin{equation}\label{eq:zempty}
Z_2^\epsilon \neq \emptyset \Rightarrow |Z^\epsilon| = 1\;.
\end{equation}
\begin{proposition}
\label{prop:Z}
$Z^\epsilon$ is a consecutive tree with root $z_{\mathrm{root}}^\epsilon$ the unique element $v \in Z^\epsilon$, such that $P(v) \notin Z^\epsilon$.
\end{proposition} 
\begin{proof}
If $Z_2^\epsilon \neq \emptyset$ by \eqref{eq:zempty} we have $|Z^\epsilon| = 1$ and the proposition is trivially verified. Hence we assume $Z^\epsilon = Z_1^\epsilon$. Consider $v \in Z^\epsilon$, such that $P(v) \notin Z^\epsilon$, there is at least one such node. We first prove that $v$ is unique. As $v \in Z^\epsilon = Z_1^\epsilon$ we know that 

\begin{equation}\label{eq:Z}
\exists k \in \{l,m,r\} : \left| \mu_{v(k)} - \tau \right| \leq \epsilon \;.
\end{equation} 
Now since $v(l), v(r) \in P(v)$ and $P(v) \notin Z^\epsilon$, it follows that, thanks to \eqref{eq:Z},

$$ \forall k \in \{l,r\} : \left| \mu_{v(k)} - \tau \right| > \epsilon  \qquad |\mu_{v(m)} - \tau| \leq \epsilon\,.$$
For node $q \neq v$ satisfying the same properties, assume that $v(m) < q(m)$ without loss of generality. With this assumption we have,

\[v(r) \leq v(m) \leq q(l) \leq q(m)\;,\]

however, as the sequence $\left( \min(|\mu_i - \tau|,\epsilon)\sign(\mu_i - \tau)\right)_{k \leq K}$ is increasing we must have $|\mu_{v(r)} - \tau| \leq \epsilon$  and $|\mu_{q(l)} - \tau| \leq \epsilon$, a contradiction. Hence $v = q$, and thus $v$ is unique which implies $\forall q \in Z^\epsilon:\ q \neq v,\, P(q) \in Z^\epsilon$.
\end{proof}

At time $t$ we define $w_t^\epsilon$ as the node of maximum depth whose subtree contains both $v_t$ and an ``$\epsilon$-good node" belonging to $Z^\epsilon$. Formally, for $t \leq T_1$,
\begin{equation*}
w_t^\epsilon := \argmax_{\{w :ST(w) \cap Z^\epsilon \neq \emptyset\; \& \;v_t \in ST(w)\}} |w|\,.
\end{equation*} 

\begin{lemma}\label{lem:d}
The node $w_t^\epsilon$ is unique and 
    \begin{equation}\label{eq:w}
    w_t^\epsilon = \argmin_{\{w :ST(w) \cap Z^\epsilon \neq \emptyset\; \& \;v_t \in ST(w)\}}\left(  \left|v_t\right| - \left|w\right| + \left(|z_{\root}^\epsilon|-|w| \right)^{+} \right) \;. 
\end{equation} 
\end{lemma}
\begin{proof}

At time $t$ consider, a node $q_t^\epsilon$ which also satisfies \ref{eq:w}, giving

$$\left|v_t\right| - \left|w_t^\epsilon\right| + \left(|z_{\root}^\epsilon|-|w_t^\epsilon| \right)^{+} = \left|v_t\right| - \left|q_t^\epsilon\right| + \left(|z_{\root}^\epsilon|-|q_t^\epsilon| \right)^{+}\;.$$

As $v_t \in ST(w_t^\epsilon)$ and $v_t \in ST(q_t^\epsilon)$ we can assume without loss of generality $q_t^\epsilon \in ST(w_t^\epsilon)$ with $|q_t^\epsilon| \geq |w_t^\epsilon|$. Thus,

$$\left|v_t\right| - \left|q_t^\epsilon\right|  \leq \left|v_t\right| - \left|w_t^\epsilon\right| \;,$$

and therefore,

$$\left(|z_{root}^\epsilon| - |q_t^\epsilon|  \right)^{+} \geq \left( |z_{root}^\epsilon| - |w_t^\epsilon| \right)^{+}\;, $$

which implies, $|q_t^\epsilon| \geq |w_t^\epsilon|$, therefore $|q_t^\epsilon| = |w_t^\epsilon|$ and as $q_t^\epsilon \in ST(w_t^\epsilon)$, we have $q_t^\epsilon =  w_t^\epsilon$.

\end{proof}
For $t\leq T_1$ we define $D_t^\epsilon$ as the distance from $v_t$ to $Z^\epsilon$, it is taken as the length of the path running from $v_t$ up to $w_t^\epsilon$ and then down to an $\epsilon$-good node in $Z^\epsilon$. Formally, we have
\[D_t^\epsilon := \left|v_t\right| - \left|w_t^\epsilon\right| + \left(|z_{\root}^\epsilon|-|w_t^\epsilon| \right)^{+}. \]
Note the following properties of $D_t^\epsilon$ and $w_t^\epsilon$,
\begin{gather}
    ST(v_t) \cap Z^\epsilon \neq \emptyset \Rightarrow v_t = w_t^\epsilon \;,\label{eq:ST} \\
    D_t = 0 \Rightarrow v_t = w_t^\epsilon \; \text{And} \; w_t^\epsilon,v_t \in Z^\epsilon \;. \label{eq:D0}
\end{gather}
Let $S_t^\epsilon$ denote the list produced by an execution of algorithm \choose with parameter $\epsilon \geq \epsilon_0$. We define $W_\epsilon$ as the set of $\epsilon$-good arms
\[
W_\epsilon := \big\{k\in [K]: \Delta_k \leq 3\epsilon \; \mathrm{OR} \; \mu_{k-1} < \tau < \mu_k \big\}\,,
\]
and at time $t$ the counter $G_t^\epsilon$, tracking the number of $3\epsilon$-good arms in $S_t^{2\epsilon}$, 
\begin{equation}\label{eq:gdef}
    G_t^\epsilon := \Big|\big\{k \in S_t^{2\epsilon}:\ k\in W_{3\epsilon}\big\}\Big| \;.
\end{equation}
Note that if $\hk$ belongs to this set then we suffer at most a regret of $3\epsilon$. We define also the favorable event where the estimates the means are close to the true ones for all the arms in $v_t$,
\begin{equation}\label{eq:gooddec}
\xi_t^\epsilon := \left\{\forall k\in \{l,m,r\}, \left|\hat{\mu}_{k,t} - \mu_{v_t(k)}\right| \leq \epsilon\right\}\,.
\end{equation}

\paragraph{Step 2: Actions of the algorithm on all iterations}\label{par:2}

After any execution of algorithm \explore and subsequent execution of algorithm \choose with parameter $\epsilon$, note the following,
\begin{itemize} 
\item for $t \leq T_1$,\; $v_t$ and $v_{t+1}$ are separated by at most one edge, i.e. 
\begin{equation}\label{eq:1edge}
v_{t+1} \in \{L(v_t), R(v_t), P(v_t) \}\,,
\end{equation} 
\item for $t \leq T_1$, 
\begin{equation} \label{eq:S}
|S_t^{2\epsilon}| \leq |S_{t+1}^{2\epsilon}| \leq |S_t^{2\epsilon}| +1 \,. 
\end{equation} 
\end{itemize}

\begin{lemma}\label{lem:allit} 
On execution of algorithm \explore and algorithm \choose with parameter $\epsilon > 0$ for all $t \leq T_1$ we have the following,
\begin{gather}
    D_{t+1}^\epsilon \leq D_t^\epsilon + 1 \label{eq:dxi}, \\
    G_{t+1}^\epsilon \geq G_t^\epsilon \label{eq:bg} \,.
\end{gather}

\end{lemma}
\begin{proof}
As the algorithm moves at most 1 step per iteration, see~\eqref{eq:1edge}, for $t \leq T_1$,  it holds
$$\left| \left|v_{t}\right| - \left|w_{t}^\epsilon\right| \right| \geq \left| \left|v_{t+1}\right| - \left|w_{t}^\epsilon \right| \right| -1\;.$$
Noting that,
\begin{align*}
D_t^\epsilon &= \left| \left|v_{t}\right| - \left|w_{t}^\epsilon\right| \right| + \left(|z_{\root}^\epsilon|-|w_t^\epsilon| \right)^{+} \\
& \geq \left| \left|v_{t+1}\right| - \left|w_{t}^\epsilon\right| \right| + \left(|z_{\root}^\epsilon|-|w_t^\epsilon|\right)^{+} -1\\
& \geq \left| \left|v_{t+1}\right| - \left|w_{t+1}^\epsilon \right| \right| + \left( |z_{\root}^\epsilon| - |w_{t+1}^\epsilon|\right)^{+} -1\\
&= D_{t+1}^\epsilon - 1\;,
\end{align*}
where the third line comes from the definition of $w_{t+1}^\epsilon$, see~\eqref{eq:w}, we obtain $D_{t+1}^\epsilon \leq D_t^\epsilon + 1$.
By~\eqref{eq:S} we have, for $t \leq T_1$, 
$$|S_t^{2\epsilon}| \leq |S_{t+1}^{2\epsilon}| \leq |S_{t}^{2\epsilon}| + 1\,, $$
hence $G_{t+1}^\epsilon \geq G_t^\epsilon$.
\end{proof}

\paragraph{Step 3: Actions of the algorithm on $\xi_t^\epsilon$}\label{par:3}

\begin{lemma}\label{lem:dxic}
On execution of algorithm \explore and algorithm \choose with parameter $\epsilon > 0$ for all $t \leq T_1$, on $\xi_t^\epsilon$, we have the following,
\begin{gather}
    D_{t+1}^\epsilon \leq \max(D_t^\epsilon - 1,0)\;,     \label{eq:dxic}\\
    G_{t+1}^\epsilon \geq G_t^\epsilon + \ind_{\{D_t^\epsilon = 0\}}\,. \label{eq:gxic} 
\end{gather}
\end{lemma}

\begin{proof}
We first prove~\eqref{eq:gxic}. Note that if the arm $v_t(k)$ is added in $S_{t+1}^{2\epsilon}$ then either $|\hmu_{k,t}-\tau|\leq 2 \epsilon$ or $v_t(k) = v_t(r) = v_t(l)+1$  and $\hmu_{l,t}+\epsilon\leq \tau \leq \hmu_{r,t}$. Thus, on $\xi_t^\epsilon$, we obtain in the first case $\Delta_{v_t(k)}\leq 3\epsilon$ and in the second case 
\[  v_t(k) = v_t(l) = v_t(r)-1 \text{ and } \mu_{v_t(l)}+\epsilon \leq \tau \leq \mu_{v_t(r)} - \epsilon\,. \]
In both case we have $v_t(k)\in W^{3\epsilon}$, hence $G_{t+1}^\epsilon \geq G_t^\epsilon +1$. It remains to prove that, when $D_t=0$, an arm is effectively added in $S_{t+1}^{2\epsilon}$.
If $D_t^\epsilon=0$ then we know $v_t\in Z^\epsilon$. If $v_t \in Z_1^\epsilon$ then under $\xi_t^\epsilon$ there exists $k\in\{l,m,r\}$ such that $|\hmu_{k,t}-\tau|\leq 2\epsilon$. Otherwise we know that 
\[
v_t(l) = v_t(r)-1\text{ and } \mu_{v_t(l)}+\epsilon \leq \tau \leq \mu_{v_t(r)}-\epsilon\,,
\]
which implies on $\xi_t^\epsilon$  that
\[
\hmu_{l,t} \leq \tau \leq \hmu_{r,t}\,.
\]
In both case an arm is added to $S_{t+1}^{2\epsilon}$.

Now we prove~\eqref{eq:dxic}. Note that on the favorable event $\xi_t^\epsilon$, we have $\forall k \in \{l,m,r\}$,
\begin{gather}\label{eq:1}
\mu_{v_t(k)} \geq \tau +\epsilon \Rightarrow \hat\mu_{k,t} \geq \tau\,, \\
\mu_{v_t(k)} \leq \tau - \epsilon \Rightarrow \hat\mu_{k,t} \leq \tau \, \label{eq:2}.
\end{gather} 
We consider the following three cases:
\begin{itemize} 

\item If $\tau \notin \left[\mu_{v_{t}(l)}+ \epsilon,\, \mu_{v_{t}(r)} - \epsilon\right]$. From~\eqref{eq:1} and~\eqref{eq:2}, under $\xi_{t}^{\epsilon}$, we get $\tau \notin \left[\hat{\mu}_{l,t},\hat{\mu}_{r,t}\right]$, and therefore $v_{t+1} = P(v_t)$. Since in this case we are getting closer to the set of $\epsilon$-good nodes by going up in the tree we know that $w_t^\epsilon = w_{t+1}^\epsilon$. Thus thanks to Lemma~\ref{lem:d}, under $\xi_t^\epsilon$,
$$  D_{t+1}^\epsilon = \left|v_{t+1}\right| - \left|w_{t+1}^\epsilon\right| + \left(|z_{\root}^\epsilon|-|w_{t+1}^\epsilon| \right)^{+} = \left|v_{t}\right| -1 - \left|w_{t}^\epsilon\right| + \left(|z_{\root}^\epsilon|-|w_{t}^\epsilon| \right)^{+} = D_t^\epsilon - 1 \,.$$

\item If $\tau \in \left[\mu_{v_{t}(l)}+ \epsilon,\,\mu_{v_{t}(r)} - \epsilon\right]$ and $v_t \notin Z^\epsilon$. Note that in this case $v_t$ can not be a leaf and we just need to go down in the subtree of $v_t$ to find an $\epsilon$-good node, id est $w_t = v_t$. Since $v_t\notin Z^\epsilon$, without loss of generality, we can assume for example $\mu_{v_t(m)} > \tau + \epsilon$. From~\eqref{eq:1} and~\eqref{eq:2}, under $\xi_t^\epsilon$, we then have $\tau\in [\hat{\mu}_{l,t},\,\hat{\mu}_{r,t}]$ and $\hat{\mu}_{m,t} \geq \tau$. Hence algorithm \explore goes to the correct subtree, $v_{t+1} = L(v_t)$. In particular we also have for this node 
$$ \tau \in \left[\mu_{v_{t+1}(l)}- \epsilon,\,\mu_{v_{t}(m)} + \epsilon\right] \,,$$
therefore it holds again $w_{t+1} = v_{t+1}$. Thus combining the previous remarks we obtain thanks to Lemma~\ref{lem:d}, under $\xi_t^\epsilon$,
$$ D_{t+1}^\epsilon = \left(|w_{t+1}| - |z_{\root}^\epsilon| \right)^{+} = \left(|w_t| - |z_{\root}^\epsilon| \right)^{+}-1 =  D_t^\epsilon - 1 \;.$$

\item If $\tau \in \left[\mu_{v_{t}(l)}+ \epsilon,\,\mu_{v_{t}(r)} - \epsilon\right]$ and $v_t \in Z^\epsilon$. We distinguish two cases: $Z_2^\epsilon$ is empty or not. In both cases we will show that, under $\xi_t^{\epsilon}$, $v_{t+1} \in Z^\epsilon$ and thus 
\[D_{t+1}^\epsilon = D_{t}^\epsilon = 0\,.\]
Hence it remains to consider these two cases:
\begin{itemize}
    \item If $Z_2^\epsilon \neq \emptyset$. Via the definition of $Z_2^\epsilon$, see~\eqref{eq:z2}, and the fact $Z_1^\epsilon = \emptyset$, $v_t$ is a leaf with $\mu_{v_t(r)} \leq \tau - \epsilon$ and $\mu_{v_t(l)} \geq  \tau + \epsilon$. Hence from~\eqref{eq:1} and~\eqref{eq:2} we have $\hat{\mu}_{l,t} \leq \tau \leq \hat{\mu}_{r,t}$. Therefore by the action of algorithm \explore we will stay in the same node $v_{t+1} = v_t$.
    
    \item Else $Z_2^\epsilon = \emptyset$. If $\mu_{v_t(m)} \in [ \tau - \epsilon, \tau + \epsilon]$, we have $R(v_t), L(v_t), P(v_t) \in Z^\epsilon$ hence trivially $v_{t+1} \in Z^\epsilon$. Else we have $\mu_{v_t(m)} \notin [ \tau - \epsilon, \tau + \epsilon]$. Without loss of generality we assume $\mu_{v_t(m)} > \tau + \epsilon$. This implies that $\mu_{v_t(r)} > \tau + \epsilon$ and since $v_t\in Z^\epsilon = Z_1^\epsilon$ it holds  $\mu_{v_t(l)} \in [ \tau - \epsilon, \tau + \epsilon]$. Thus, under $\xi_t^c$ we then get as previously  $\tau\in [\hat{\mu}_{l,t},\,\hat{\mu}_{r,t}]$ and $\hat{\mu}_{m,t} \geq \tau$. Therefore by the action of algorithm \explore we will go to the left child $v_{t+1} = L(v_t)\in Z^\epsilon$.
\end{itemize}
\end{itemize}
\end{proof}

\paragraph{Step 4: Lower bound on $G_{T_1+1}^\epsilon$}
We denote by $\bxi_t^\epsilon$ the complement of $\xi_t^\epsilon$.
\begin{lemma}\label{lem:lb_on_Gt} For any execution of algorithm \explore and subsequent execution of \choose with parameter $\epsilon \geq \epsilon_0$,
\[
G_{T_1+1}^{\epsilon} \geq \frac{3}{4} T_1 - 2 \sum_{t=1}^{T_1}\ind_{\bxi_t^\epsilon}\,.
\]
\end{lemma}
\begin{proof}
Combining~\eqref{eq:dxic} and~\eqref{eq:dxi} from Lemma~\ref{lem:allit} and Lemma~\ref{lem:dxic} respectively we have
\begin{align*}
D_{t+1}^\epsilon &\leq D_t^\epsilon+\ind_{\bxi_t^\epsilon}-\ind_{\xi_t^\epsilon}\ind_{\{D_t^\epsilon>0\}}\\
&= D_t^\epsilon+2\ind_{\bxi_t^\epsilon}-1+\ind_{\xi_t^\epsilon}\ind_{\{D_t^\epsilon=0\}}\,.
\end{align*}
Using this inequality with~\eqref{eq:gxic} we obtain 
\begin{align*}
G_{T_1+1}^\epsilon &= \sum_{t=1}^{T_1} G_{t+1}^\epsilon-G_{t}^\epsilon \\
&\geq \sum_{t=1}^{T_1}\ind_{\xi_t^\epsilon}\ind_{\{D_t^\epsilon=0\}}\\
&\geq \sum_{t=1}^{T_1} \big(D_{t+1}^\epsilon-D_t^\epsilon -2\ind_{\xi_t^\epsilon}+1\big)\\
&\geq T_1-D_1^\epsilon-2\sum_{t=1}^{T_1}\ind_{\xi_{t,\epsilon}}\\
&\geq \frac{3}{4} T_1 -2\sum_{t=1}^{T_1}\ind_{\xi_{t,\epsilon}}\,,
\end{align*}
where we used in the last inequality the fact that $D_1\leq \log_2(K)$ and that $\log_2(K) \leq T_1/4$ by definition of $T_1$ .
\end{proof}

\paragraph{Step 5: First high probability bound on the regret}
\begin{lemma}\label{lem:fisrt_high_prb_regret_bound}
For all $\epsilon\geq \epsilon_0$, following the execution of algorithm \stb,
\begin{equation}
    \label{eq:fisrt_high_prb_regret_bound}
    \PP( R_T > 3 \epsilon) \leq e^{-3\log(K)/4}\,.
\end{equation}
\end{lemma}
Before proving Lemma~\ref{lem:fisrt_high_prb_regret_bound} we need to show that the number of times a favorable events $\xi_t^{\epsilon_0}$ occurs is not to small with high probability. Precisely in the following lemma we upper bound the probability of the event
\[
\xi^{\epsilon_0} = \left\{ \sum_{t=1}^{T_1} \ind_{\bxi_t^{\epsilon_0}} \leq \frac{T_1}{8} \right\}\,.
\]
\begin{lemma}\label{lem:proba_xi_epsilon_zero}
For any execution of algorithm \explore and subsequent execution of \choose with parameter $\epsilon_0$,
\[
\PP( \bxi^{\epsilon_0} ) \leq e^{-3 \log(K)/4}\,.
\]
\end{lemma}
\begin{proof}
Let $\cF_t$ be the information available at and including time $t$. Thanks to the Hoeffding inequality and the choice of $T_2$, we have for all $k\in \{l,m,r\}$,
\begin{equation*} 
\mathbb{P}\left( \left|\hat{\mu}_{k,t} - \mu_{v_t(k)}\right| \geq  \epsilon_0 |\cF_{t-1} \right) \leq 2\exp\!\left(-\frac{T_2\epsilon_0^2}{2\sigma^2}\right) \leq \frac{1}{24}\;,\label{eq:proba} 
\end{equation*} 
hence by a union bound $\PP(\bxi_t^{\epsilon_0}|\cF_{t-1})\leq 1/ 8$. Then the Azuma-Hoeffding inequality applied to the martingale 
\[\sum_{t=1}^{T_1} \left[\ind_{\bxi_t^{\epsilon_0}} - \PP(\bxi_t^{\epsilon_0}|\cF_{t-1})\right]\,,\] 
with respect to the filtration $(\cF_t)_{t\leq T_1}$ allows us to conclude
\begin{equation}\label{eq:hoef}
\mathbb{P}\left( \sum_{t=1}^{T_1} \left[\ind_{\bxi_t^{\epsilon_0}} - \PP(\bxi_t^{\epsilon_0}|\cF_{t-1}) \right]\geq  \frac{T_1}{4} \right) \leq e^{-2 T_1 /16} \leq e^{-3\log(K)/4}\;,
\end{equation} 
where we used that $T_1 = \lceil 6 \log(K) \rceil$.
\end{proof}
We are now ready to prove Lemma~\ref{lem:fisrt_high_prb_regret_bound}.
\begin{proof}[Proof of Lemma~\ref{lem:fisrt_high_prb_regret_bound}] We first prove it for $\epsilon = \epsilon_0$. Thanks to Lemma~\ref{lem:lb_on_Gt} on the event $\xi^{\epsilon_0}$ we have 
\[
G_{T_1+1}^{\epsilon_0} \geq \frac{3}{4} T_1 - 2 \sum_{t=1}^{T_1}\ind_{\bxi_t^{\epsilon_0}} \geq \frac{T_1}{2}\,.
\]
But thanks to the choice of $\hepsilon \geq 2\epsilon_0$ we know that 
\[
S_{T_1+1}^{2\epsilon_0} \subset S_{T_1+1}^{\hepsilon}\,.
\]
Thus there is more than the half of the arms of $S_{T_1+1}^{\hepsilon}$ in $W_{3\epsilon_0}$, since this list is at most of size $T_1$. In particular this implies that $\hk = \median(S_{T_1+1}^{\hepsilon})\in W_{3\epsilon_0}$. Indeed $W_{3\epsilon_0}$ is a segment in $[K]$, see~$\eqref{eq:Z}$. Therefore, on the event $\xi^{\epsilon_0}$ we have 
\[
R_T \leq 3\epsilon_0.
\]
Lemma~\ref{lem:proba_xi_epsilon_zero} allows us to conclude, for $\epsilon\geq \epsilon_0$,
\[
\PP( R_T > 3\epsilon ) \leq \PP( R_T > 3\epsilon_0) \leq e^{-3\log(K)/4}\,.
\]
\end{proof}

\paragraph{Step 6: Second high probability bound on the regret}
\begin{lemma}\label{lem:second_high_prb_regret_bound}
For all $\epsilon\geq \epsilon_0$, following the execution of algorithm \stb,
\begin{equation}
    \label{eq:second_high_prb_regret_bound}
    \PP( R_T > 3 \epsilon) \leq  72 \log(K) \exp\!\left(-\frac{T\epsilon^2}{36 \sigma^2 \log(K)}\right)\,.
\end{equation}
\end{lemma}
\begin{proof}
We consider the event where all the favorable events $\xi_t^\epsilon$ occur,
\[
\xi^\epsilon_a := \bigcap_{t=1}^{T_1} \xi_t^\epsilon\,.
\]
On this event $\xi^\epsilon_a$ thanks to Lemma~\ref{lem:lb_on_Gt} we have
\begin{align*}
G_{T_1+1}^{\epsilon} &\geq \frac{3}{4} T_1 - 2 \sum_{t=1}^{T_1}\ind_{\bxi_t^\epsilon}\\
&=  \frac{3}{4} T_1\,,
\end{align*}
hence $S_{T_1+1}^{2\epsilon}\neq \emptyset$ is not empty. Furthermore following the same arguments of the beginning of the proof of Lemma~\ref{lem:dxic} all arms in the list $S_{T_1+1}^{2\epsilon}\neq \emptyset$ are also in $W_{3\epsilon}$.
Then noting that by construction 
\[
\hepsilon = \inf_{\epsilon' \geq 2\epsilon_0:\ S_{T_1+1}^{\epsilon'} \neq \emptyset} \epsilon'
\,,\]
we get $\hepsilon\leq 2\epsilon$ therefore $ S_{T_1+1}^{\hepsilon} \subset S_{T_1+1}^{2\epsilon}$. Thanks to the remarks above we know that $\hk \in W_{3\epsilon}$ thus on $\xi^\epsilon_a$,
\[
R_T \leq 3\epsilon\,.
\]
The Hoeffding inequality in combination with a union bound allows us to conclude,
\begin{align}
    \PP(\bxi^\epsilon_a) \leq \sum_{t=1}^{T_1} \EE\big[\PP(\bxi_t^\epsilon|\cF_{t-1})\big] &\leq 72 \log(K) \exp\!\left(-\frac{T_2\epsilon^2}{2\sigma^2}\right)\\
    &\leq  72 \log(K) \exp\!\left(-\frac{T\epsilon^2}{36\sigma^2 \log(K)}\right)\,.
\end{align}

\end{proof}

\paragraph{Conclusion} The proof of Proposition~\ref{prop:mo_hiprob_up} is straightforward combining Lemma~\ref{lem:fisrt_high_prb_regret_bound} and Lemma~\ref{lem:second_high_prb_regret_bound}. Thus we obtain for all $\epsilon \geq 3\epsilon_0$,
\[
\PP(R_T \geq \epsilon ) \leq \min\left( \exp\!\left(-\frac{3\log(K)}{4}\right) , 72 \log(K) \exp\!\left(-\frac{T\epsilon^2}{324 \sigma^2 \log(K)}\right) \right)\,.
\]
We can integrate the high probability upper bound obtained in Proposition~\ref{prop:mo_hiprob_up} to prove Corollary~\ref{cor:mo_exp_up}.
\begin{proof}[Proof of Corollary~\ref{cor:mo_exp_up}.] Thanks to Proposition~\ref{prop:mo_hiprob_up}, for $\epsilon_1  = \log(72\log(K))\sqrt{324 \sigma^2 \log(K) / T}$, we have
\begin{align*}
\EE[R_T] &\leq  \epsilon_0 + (\epsilon_1 -\epsilon_0) e^{-3\log(K)/4} +\int_{\epsilon=\epsilon_1}^{+\infty} 72\log(K) \exp\!\left(-\frac{T\epsilon^2}{324 \sigma^2 \log(K)}\right)\\
&\leq \sqrt{\frac{36\sigma^2 \log(48)\log(K)}{T}}+ \left( \underbrace{\frac{\log(72\log(K))}{K^{3/4}}}_{\leq 3}  +\frac{\sqrt{\pi}}{2}\right)\sqrt{\frac{324 \sigma^2 \log(K)}{T}}\\
&\leq 80 \sqrt{\frac{\sigma^2 \log(K)}{T}}\,.
\end{align*}
\end{proof}
Setting $\sigma = 1$ Theorem \ref{thm:mon_minmax} follows directly from Proposition \ref{prop:mon_lo} and Corollary \ref{cor:mo_exp_up}.

\section{Proof of Theorem \ref{thm:un_minmax}}\label{app:un_minimax} 
To prove Theorem \ref{thm:un_minmax} we first demonstrate, in Proposition \ref{prop:unimod_lo}, a lower bound on the expected regret of any strategy on the \utbp. We will then show, with Proposition \ref{prop:exp_uni_hi}, that the \UTB achieves said lower bound. The proof of Theorem \ref{thm:un_minmax} will then follow directly. For all proofs during this section we make the assumption that arms are distributed as \texorpdfstring{$\sigma^2$}{TEXT}-sub-Gaussian with $\sigma = 1$. Also, we explain in the proof of the lower bound how it is possible to straightforwardly adapt the proof to the case where the distributions are supported in $[0,1]$.

\begin{proposition}\label{prop:unimod_lo}
For any $T\geq 1$ and any strategy $\pi$, there exists an unimodal bandit problem $\unu \in \Bu$, such that
\[
\bR^{\pi,\unu}_T \geq \frac{1}{8}\sqrt{\frac{K}{T}}\,.
\]
\end{proposition} 

\begin{proof}
We will proceed as in the proof of Proposition~\ref{prop:unstr_lo}. Fix some positive real number $0 < \epsilon < 1$. Without loss of generality we can assume that $\tau = \epsilon/2$. And consider the family of Gaussian bandit problems $\unu^k$ indexed by $k\in\{0,\ldots,K\}$, such that for all $k \in\{0,\ldots,K\}$, $l \in [K]$,
\[
\nu_l^k=
    \begin{cases}
      \Ng(\epsilon, \sigma^2) & \text{if}\ k = l \\
      \Ng(0, \sigma^2) & \text{else}
    \end{cases}\,.
\]
Note that if we wish to consider distributions in $[0,1]$ we can consider instead $\tau = 1/2+\epsilon/2$
\[
\nu_l^k=
    \begin{cases}
      \mathcal B(1/2+\epsilon) & \text{if}\ k = l \\
      \mathcal B(1/2) & \text{else}
    \end{cases}\,,
\]
up to minor alterations of the constants, and to considering $\tau = 1/2$.

Note that all these bandit problems belong to the set of unimodal bandit problems, $\unu^k\in\Bu$. Following the same steps as in the proof of Proposition~\ref{prop:unstr_lo} one can lower bound the maximum of the expected regrets over all the bandit problems introduced above,
\[
\max_{k\in[K]}\bR_T^{\unu^k,\pi} \geq \frac{\epsilon}{2} \left( 1- \frac{1}{K}\sum_{k=1}^K \EE_k \ind_{\{\hQ=Q^k\}}\right)\,, 
\]
where we denote by $\EE^k$ the expectation and by $Q^k$ the true classification in the problem $\unu^k$. Thanks to the contraction and the convexity of the relative entropy we have 
\begin{align*}
\kl\!\left(\underbrace{\frac{1}{K}\sum_{k=1}^K \EE_0 \ind_{\{\hQ=Q^k\}}}_{\leq 1/K} ,\frac{1}{K}\sum_{k=1}^K \EE_{k} \ind_{\{\hQ=Q^k\}}\right) &\leq \frac{1}{K}\sum_{k=1}^K\EE_{0}\!\big[N_k(T)\big] \frac{\epsilon^2}{2 \sigma^2}\\
&\leq \frac{T \epsilon^2}{2 K \sigma^2}\,.
\end{align*}
Then using the Pinsker inequality $\kl(x,y)\geq 2(x-y)^2$, we obtain
\begin{equation*}
\frac{1}{K}\sum_{k=1}^K \EE_{k} \ind_{\{\hQ=Q^k\}} \leq \frac{1}{K} + \sqrt{\frac{ T\epsilon^2}{4\sigma^2K}}\,.
\end{equation*}
Hence combining the last three inequalities we get 
\[
\max_{k\in[K]}\bR_T^{\unu^k,\pi} \geq \frac{\epsilon}{2} \left( \frac{1}{2}- \sqrt{ \frac{T \epsilon^2}{4\sigma^2 K}}\right)\,.
\]
Choosing $\epsilon = \sqrt{4\sigma^2  K/ T}$ allows us to conclude.
\end{proof}

\begin{proposition}
\label{prop:exp_uni_hi}
There exists a universal constant $c_{\mathrm{uni}}>0$ such that for any unimodal bandit problem $\unu \in \Bu$, \UTB satisfies \vspace{-1mm}
$$\bR_T^{\ctb, \unu} \leq  c_{\mathrm{uni}}\sqrt{\frac{ K}{n}}\;.\vspace{-3mm}$$

\end{proposition}
\begin{proof}

{\bf Step 1: Definitions} Write
$$\hat \Delta = \mu^* - \mu_{\hat m},$$
and
$$\hat \epsilon = |\hmu_{\hat l} - \mu_{\hat l}| \lor |\hmu_{\hat r} - \mu_{\hat r}|\lor |\hmu_{\hat m} - \mu_{\hat m}| \lor |\hmu_{\hat r+1} - \mu_{\hat r+1}| \lor |\hmu_{\hat l-1} - \mu_{\hat l-1}|.$$
and we write $R^{(l)}$ for the regret of \stb on $\{1,\ldots, \hat m\}$ when played by algorithm \UTB, and $R^{(r)}$ for the regret of \dstb on $\{ \hat m, \ldots, K\}$ when played by algorithm \UTB.
Let us also write $R_T = R_T^{\UTB,\nu}$ for the regret associated to the outputted set $\hat S$.
$$\mathcal E^{(l)} = \{|\hmu_{\hat l}-\tau| \leq \hmu_{\hat m} - \tau \} \cup \{\hmu_{\hat l-1}\leq \tau \leq\hmu_{\hat l}\},$$
and define similarly $\mathcal E^{(r)}$ replacing $l$ by $r$. Define
$$\mathcal E = \{\mathcal E^{(l)} \cap \mathcal E^{(r)}\}.$$

\textbf{ Step 2: Bound on the regret on the events}
Assume without loss of generality that $R^{(l)} \geq R^{(r)}$. By definition of the algorithm this implies under this condition that
$$R_T = R^{(l)}  \ind_{\{\mathcal E\}} + (\mu^* - \tau)_+ \mathbf \ind_{\{\mathcal E^C\}},$$
which implies directly
\begin{equation}\label{eq:RTdec}
R_T \leq R^{(l)} \mathbf \ind_{\{\mathcal E   \}} + (\mu_{\hat m} - \tau)_+ \mathbf 1\{\mathcal E^C  \} + \hat \Delta .    
\end{equation}
Note that
$$\mathcal E \subset \{|\mu_{\hat l}-\tau| \leq \mu_{\hat m} - \tau  + 2 \hat \epsilon\} \cup \{\mu_{\hat l-1} - \hat \epsilon  \leq \tau \leq \mu_{\hat l} + \hat \epsilon \}.$$
And so since $R^{(l)} \leq |\mu_{\hat l}-\tau|$, we have 
\begin{equation}\label{eq:RTvombi1}
     R^{(l)} \mathbf \ind_{\{\mathcal E   \}} \leq R^{(l)}\land (\mu_{\hat m} - \tau)_+ + 2\hat \epsilon.
\end{equation}
Note also that on $\mathcal E^C$ and under our condition $R^{(l)} \geq R^{(r)}$, we have that
$$\mathcal E^C \cap \{R^{(l)} \geq R^{(r)}\} \subset \{|\mu_{\hat l}-\tau|  \geq \mu_{\hat m} - \tau - 2 \hat \epsilon\}.$$
And on $\mathcal E^C \cap \{R^{(l)} \geq R^{(r)}\}$, we have that $R^{(l)} \geq (\mu_{\hat l}-\tau)_+ - 2\hat \epsilon$, which leads to under our assumption $R^{(l)} \geq R^{(r)}$
\begin{equation}\label{eq:RTvombi2}
(\mu_{\hat m} - \tau)_+ \mathbf \ind_{\{\mathcal E^C  \}}  \leq R^{(l)} \land (\mu_{\hat m} - \tau)_+ + 2\hat \epsilon .
\end{equation}
So we have combining \eqref{eq:RTvombi1} and \eqref{eq:RTvombi2} all cases in \eqref{eq:RTdec} that if $R^{(l)} \geq R^{(r)}$
$$R_T \leq (R^{(l)} ) \land (\mu_{\hat m} - \tau)_+  + 2\hat \epsilon + \hat \Delta.$$
Considering similarly the case $R^{(r)} \geq R^{(l)}$ gives
$$R_T \leq (R^{(l)} \lor R^{(r)} ) \land (\mu_{\hat m} - \tau)_+  + 2\hat \epsilon + \hat \Delta.$$

{\bf Step 3: Integration of the regret}
Consider $\epsilon_0 = 4 c_{\SR}\sqrt{\frac{K}{n}}$. Consider the event where $(\mu_{\hat m} - \tau)_+ = \tilde \epsilon \geq \epsilon_0$. On this event, and since the sequence of arms's means is unimodal, \stb satisfies the assumptions of Corollary~\ref{cor:hireg} for $\tilde \epsilon$ and a set of arms $\{1, \ldots, \hat m\}$, and integrating over the tail probability between $\epsilon_0$ and $\tilde \epsilon$ - conditional  to  we know that there exists an absolute constant $C>0$ such that
$$\mathbb E[R^{(l)} \land \tilde \epsilon| (\mu_{\hat m} - \tau)_+ = \tilde \epsilon] \leq C \sqrt{\frac{\log K + 1}{n}}.$$
Similarly
$$\mathbb E[R^{(r)} \land \tilde \epsilon| (\mu_{\hat m} - \tau)_+ = \tilde \epsilon] \leq C \sqrt{\frac{\log K + 1}{n}}.$$
And so
$$\mathbb E\Big[(R^{(l)} \lor R^{(r)} ) \land (\mu_{\hat m} - \tau)_+ \Big] \leq C \sqrt{\frac{\log K + 1}{n}}.$$
combining this with the sub-Gaussian properties of the means which give that
$$\mathbb E \hat \epsilon \leq c \sqrt{\frac{1}{T}},$$
where $c>0$ is some absolute constant, and with the minimax optimality of \SR which gives
$$\mathbb E \hat \Delta \leq  4c_{\SR}\sqrt{\frac{K}{T}},$$
this provides the result.

\end{proof}

\section{Proof of Theorem \ref{thm:con_minmax}}\label{app:con_minimax}
For the proof of Proposition \ref{prop:exp_con_hi} we make the assumption that the distribution of all arms is bounded on the $[0,1]$ interval. In the case of the lower bound we consider \texorpdfstring{$\sigma^2$}{TEXT}-sub-Gaussian distributions. Also, we explain in the proof of the lower bound how it is possible to straightforwardly adapt the proof to the case where the distributions are supported in $[0,1]$.

\begin{proposition}\label{prop:lb_cvx}
 For any $T\geq 1$, $K\geq e^{12}$ and any strategy $\pi$, there exists a structured bandit problem $\unu \in \Bc$, such that
\[
\bR^{\pi,\unu}_T \geq \frac{1}{8}\sqrt{\frac{\sigma^2 \max\big(2,\log(\log(K)-1)\big)}{8T}}\,.
\]
\end{proposition}

\begin{proof}
We will proceed as in the previous proofs but with a different alternative set. Fix some positive real number $\epsilon$ in $[0,1]$ and without loss of generality set $\tau = \epsilon$. And consider the family of Gaussian bandit problems $\unu^l$ indexed by $l\in\{0,\ldots,L:=\lfloor \log_2(K) \rfloor\}$  defined by $\unu^l = \Ng(\mu^l,1)$ with 
\[
\mu_k^l=
    \begin{cases}
      \frac{k}{k_l}\epsilon &\text{ if } k\leq 2 k_l:=  2^{l+1}\\
      2\epsilon &\text{else}\,.
      
    \end{cases}\,.
\]
Note that if we want to consider distributions supported in $[0,1]$ we can consider $\unu^l_k = \mathcal B(1/2+\mu^l_k)$ and $\tau =1/2+\epsilon$ instead of the Gaussian distributions, up to minor adaptations of the constants, and to considering $\tau = 1/2+\epsilon$.

Note that all these bandit problems belong to the set of convex bandit problems, $\unu^k\in\Bc$. We will lower bound the maximum of the expected regrets over all the bandit problems introduced above,
\[
\max_{l\in[L]}\bR_T^{\unu^l,\pi} =\max_{l\in[L]} \EE_l\!\left[\max_{k\in[K]}\Delta_k^l \ind_{\{\hQ_k\neq Q_k^l\}}\right]\,, 
\]
where we denote by $\EE^l$ the expectation and by $Q^l$ the true classification in the problem $\unu^l$. In particular we have $Q^l = [-1,\ldots,-1,1,\ldots,1]$ where the first one is at position $k_l$. 
Let $\hl = \argmin\{j\in[L]:\ \forall i \geq j,\ \hQ_{k_i} = 1\}$ be an estimate for the index of the problem with the convention $\hl=L$ if the set is empty. Then we have 
\[
\max_{k\in[K]}\Delta_k^l \ind_{\{\tQ_k \neq Q_k^l\}} \geq \frac{\epsilon}{2} \ind_{\big\{\hl \notin\{l,l+1\}\big\}}\,.
\]
Indeed if $\hl<l$ then we know that $\hQ_{k_{\hl}} = 1\neq -1 = Q_{k_{\hl}}$, thus we obtain
\[
\max_{k\in[K]}\Delta_k^l \ind_{\{\hQ_k \neq Q_k^l\}} \geq \Delta_{k_{\hl}}^l = \epsilon - \frac{k_{\hl}}{k_l}\epsilon \geq \frac{\epsilon}{2}\,.
\]
Else $\hl>l+1$, and similarly we get, because $\hQ_{k_{\hl-1}} = -1 $ and $\hl-1>l$ (or $\hQ_{k}=-1$ for some $k>\hl$ if we choose $\hl=L$ in the case where the set defining $\hl$ is empty),
\[
\max_{k\in[K]}\Delta_k^l \ind_{\{\hQ_k\neq Q_k^l\}} \geq \Delta_{k_{\hl-1}}^l =\min\!\left(\frac{k_{\hl-1}}{k_l}\epsilon - \epsilon ,\epsilon\right)\geq \frac{\epsilon}{2}\,.
\]
Using the previous inequality we obtain 
\[
\max_{l\in[L]}\bR_T^{\unu^l,\pi} \geq \frac{\epsilon}{2} \max_{l\in[L]}\EE_l [1-\ind_{\{\hl=l\}}- \ind_{\{\hl=l+1\}}]\geq  \frac{\epsilon}{2}\left(1 -\frac{2}{L}\sum_{l\in [L]} \EE_l\ind_{\{\hl = l\}} \right)\,.
\]
We can conclude as previously. Thanks to the contraction and the convexity of the relative entropy we have 
\begin{align*}
\kl\!\left(\frac{1}{L}\sum_{l=1}^L \EE_{l} \ind_{\{\hl=l\}},\underbrace{\frac{1}{L}\sum_{l=1}^L \EE_0 \ind_{\{\hl=l\}}}_{\leq 1/L}\right) &\leq \frac{1}{L}\sum_{l=1}^L  \sum_{k=1}^K\EE_{l}\!\big[N_k(T)\big] \frac{\epsilon^2}{2 \sigma^2}\\
&\leq \frac{T \epsilon^2}{2 \sigma^2}\,.
\end{align*}
Then using a refined Pinsker inequality $\kl(x,y)\geq (x-y)^2 \max\big(2,\log(1/y)\big)$, we obtain
\begin{equation*}
\frac{1}{L}\sum_{l=1}^L \EE_{l} \ind_{\{\hl=l\}} \leq \frac{1}{L} + \sqrt{\frac{ T\epsilon^2}{2\sigma^2 \max\big(2,\log(L)\big)}}\,.
\end{equation*}
Hence combining the last three inequalities we get 
\[
\max_{l\in[L]}\bR_T^{\unu^l,\pi} \geq \frac{\epsilon}{2} \left(1- \frac{2}{L}- 2\sqrt{ \frac{T \epsilon^2}{2\sigma^2 \max\big(2,\log(L)\big)}}\right)\,.
\]
Choosing $\epsilon = \sqrt{\sigma^2 \max\big(2,\log(L)\big)/(8 T)}$ allows us to conclude.
\end{proof}

\begin{proposition}\label{prop:exp_con_hi}
There exists a universal constant $c_{\mathrm{conv}}>0$ such that for any convex bandit problem $\unu \in \Bc$, \CTB satisfies \vspace{-1mm}
$$\bR_T^{\ctb, \unu} \leq  c_{\mathrm{conv}}\sqrt{\frac{\log\log K \lor 1}{T}}\;.$$
\end{proposition}
Before going on to prove Proposition \ref{prop:exp_con_hi} we first show the following. 

\begin{lemma}\label{lem:con_analytic_1}
Consider $1 \leq p \leq q \leq K$, $\tilde \epsilon>0, \tilde \tau \in \mathbb R$. Consider any $1 \leq p <q \leq K$, such that, 
\begin{equation}\label{eq:middlarge}
    \mu_{\lfloor \frac{p+q}{2}\rfloor} \geq \tilde \tau + \frac{1}{8}\tilde \epsilon \;.
\end{equation}

Then
 \begin{equation*} 
\left( \min(|\mu_k - \tilde \tau|,\frac{1}{8}\tilde \epsilon)\sign(\mu_k - \tilde \tau)\right)_{k}\;,
\end{equation*} 
is monotonically increasing on $ [p:\lfloor\frac{p+q}{2}\rfloor]$ and monotonically decreasing on $[\lfloor\frac{p+q}{2}\rfloor: q]$.

\end{lemma} 

\begin{proof}
We just prove that the sequence is monotonically increasing on $ [p:\lfloor\frac{p+q}{2}\rfloor]$, the other case is proven similarly.

Since $(\mu_k)_{k\leq K}$ is concave, we know that there exists $k^*\in \{1,\ldots, K\}$ such that $(\mu_k)_{k\leq k^*}$ is increasing and $(\mu_k)_{k\geq k^*}$ is decreasing. 
\begin{itemize}
\item If $k^* \in [p,\lfloor\frac{p+q}{2}\rfloor]$, and since \eqref{eq:middlarge} holds, we have that $\forall k \in [k^*,\lfloor\frac{p+q}{2}\rfloor]$, $\mu_k - \tilde \tau \geq \tilde \epsilon/8$. This implies the result.
\item If $k^* \not\in [p:\lfloor\frac{p+q}{2}\rfloor]$, we have either (i) that $\mu_k$ is increasing on the interval which implies the result or (ii) that $\mu_k$ is decreasing on the interval. In case (ii), we know by \eqref{eq:middlarge} that $\forall k \in [p,\lfloor\frac{p+q}{2}\rfloor]$, $\mu_k - \tilde \tau \geq \tilde \epsilon/8$. This implies the result.
\end{itemize}

\end{proof}
\begin{lemma}\label{lem:con_analyitc2}
Let $\tilde \epsilon>0, \tilde \tau \in \mathbb R$. For any $1  \leq p \leq q \leq K$, such that, 
\begin{gather*}
\mu_{p} \land \mu_{q} \geq \tilde \tau -  \tilde \epsilon\;,\\
\mu_{\lfloor \frac{p+q}{2} \rfloor } \leq \tilde \tau - \frac{5}{8} \tilde \epsilon \;,
\end{gather*}
we have that, $\forall k \in \{p,\ldots,q\}$ that $\mu_k \leq \tilde \tau - \frac{1}{8}\epsilon$.
\end{lemma}
\begin{proof}

We assume $\exists k \in \{p,\ldots, q\}$ such that $\mu_k > \tau - \frac{1}{8} \tilde \epsilon$ and aim to prove by contradiction.  Without loss of generality assume $k < \frac{p+q}{2}$, in combination with the assumptions of Lemma~\ref{lem:con_analyitc2} we have $(\mu_k - \mu_{\lfloor\frac{p+q}{2}\rfloor})  > \frac{1}{2}\tilde \epsilon$. However, via the convex property $(\mu_k - \mu_{\lfloor \frac{p+q}{2} \rfloor})  \leq (\mu_{\lfloor \frac{p+q}{2} \rfloor} - \mu_{q})$, a contradiction as it implies with the forelast equation that $\mu_{q} < \tilde \tau - \frac{1}{8}\tilde \epsilon$. 
\end{proof}

We now define the event, 

\begin{align*}\label{lem:con_good_conditions}
\xi_i &:= \Big(\xi_i^{(L)} \cap \xi_i^{(R)}\Big)\cup \xi_i^{(A)}\\
&:= \Bigg(\Bigg\{ \mu_{l_i} \geq \tau - \epsilon_i\;, \forall k < l_i : \mu_k \leq \tau - \frac{1}{2} \epsilon_i\;\Bigg\}\\
&\cap \Bigg\{\mu_{r_i} \geq \tau - \epsilon_i\;,\forall k > r_i: \mu_k \leq  \tau - \frac{1}{2} \epsilon_i\;\Bigg\}\Bigg)\\
&\cup \Bigg\{ \forall k\leq K, \mu_k \leq \tau - \frac{1}{8} \epsilon_i\Bigg\}.
\end{align*}

Consider the event 
\begin{equation}
\mathcal E_i = \{\mu_{m_i} \geq \tau_i + \frac{1}{8} \epsilon_i\}.
\end{equation}

\begin{proposition}\label{prop:applystb}
Let $i\leq M$ and set
$$\delta_i' = \min\left( \exp\!\Bigg(-\frac{3\log\log(K)}{4}\right) ,\, 72 \log\log(K) \exp\!\left(-\frac{T_2^{(i)}\epsilon_i^2}{216 \times 64 \log\log(K)}\right) \Bigg).$$
Let $l'_{i+1}$ be the largest arm smaller than $l_{i+1}$ in $\mathcal S^{\log}_{l_i, r_i}$. It holds that
$$\mathbb P\Bigg(|\mu_{l_{i+1}}-\tau_i| \leq \epsilon_i/8~~OR~~\mu_{l'_{i+1}}+ \epsilon_i/8 < \tau_i < \mu_{l_{i+1}}- \epsilon_i/8   \Bigg| \mathcal E_i\Bigg) \geq 1-\delta'_i.$$
Also for $r'_{i+1}$ be the smallest arm smaller than $r_i$ in $- \mathcal S^{\log}_{-r_i, l_i}$.
$$\mathbb P\Bigg(|\mu_{r_{i+1}}-\tau_i| \leq \epsilon_i/8~~OR~~\mu_{r'_{i+1}}+ \epsilon_i/8 < \tau_i < \mu_{r_{i+1}}- \epsilon_i/8 \Bigg| \mathcal E_i\Bigg) \geq 1 - \delta'_i.$$
\end{proposition}
\begin{proof}
A straightforward corollary of Proposition~\ref{prop:mo_hiprob_up} is as follows.

\begin{corollary}\label{cor:hireg}
Consider a problem $\unu \in \cB(K)$ and $\epsilon \geq \sqrt{\frac{ 2\log(48)6\log(K)}{T}}$, such that $(\min(|\mu_k - \tau|, \epsilon) \mathrm{sign}(\mu_k - \tau) +\tau)_k$ is increasing with $k$. Then the \stb Algorithm will allow us to identify and arm $\hk$ such that, 
\[|\mu_{\hk} - \tau| \leq \epsilon \; \OR \; \mu_{\hk-1} +\epsilon \leq \tau \leq \mu_{\hk-1} - \epsilon\]

 with probability greater than, 
\begin{equation*}
   1 - \min\left( \exp\!\Bigg(-\frac{3\log(K)}{4}\right),\,
72 \log(K) \exp\!\left(-\frac{T\epsilon^2}{216  \log(K)}\right) \Bigg)\,. 
\end{equation*}
\end{corollary}

The result of the proposition follows by applying this corollary and noting that
\begin{itemize}
\item in any case, $|\mathcal S^{\log}_{l_i, r_i}| \leq \log K$ so that we apply $\stb$ on a problem that has less than $\log K$ arms,
\item that on $\mathcal E_i$, we have that $(\min(|\mu_k - \tau_i|,  \epsilon_i/8) \mathrm{sign}(\mu_k - \tau_i))_{k \in [l_i, m_i]}$ is increasing (respectively, $(\min(|\mu_k - \tau_i|,  \epsilon_i/8) \mathrm{sign}(\mu_k - \tau_i))_{k \in [m_i, r_i]}$ is decreasing) - see Lemma~\ref{lem:con_analytic_1}. 
\item Moreover $\epsilon_i \geq \epsilon_M \geq \sqrt{\frac{2\log(48)6\log\log(K)}{T}}$. And so since $\mathcal S^{\log}_{l_i, r_i} \subset [l_i, m_i]$ (resp.~$-\mathcal S^{\log}_{-r_i, -l_i} \subset [m_i, r_i]$) and $\mathcal |S^{\log}_{l_i, r_i}| \leq \log(K)$, the conditions of Corollary~\ref{cor:hireg} are satisfied, for the set $\mathcal S^{\log}_{l_i, r_i}$ of arms.
\end{itemize}

Therefore we can apply Corollary \ref{cor:hireg} to show that when running \stb($\mathcal S^{\log}_{l_i,r_i}$, $\tau_i$, $T_2^{(i)}$) we are able to identify an arm $\hk$ such that setting $l_{i+1} = \hk$ satisfies our result with probability greater than $1 - \delta_i'$.

\end{proof}

\begin{proposition}\label{prop:xiLR}
We have that for $i \leq M$
$$\mathbb P\Big(\xi_{i+1}^{(L)}\Big| \xi_i  \cap \mathcal E_i\Big) \geq 1 - \delta'_i,$$
and
$$\mathbb P\Big(\xi_{i+1}^{(R)}\Big| \xi_i  \cap \mathcal E_i\Big) \geq 1 - \delta'_i.$$
\end{proposition}
\begin{proof}

We prove this proposition only for $\xi_{i+1}^{(L)}$ as the proof for $\xi_{i+1}^{(R)}$ is similar. Consider the high probability event of Proposition~\ref{prop:applystb}, where we just have two possibilities for the mean of $l_{i+1}$ which we summarize below.

\paragraph{Case 1} Consider the case where $\stb$ outputs $l_{i+1}$ such that,
\begin{equation} 
\mu_{l'_{i+1}}+ \epsilon_i/8 < \tau_i < \mu_{l_{i+1}}- \epsilon_i/8\;,
\end{equation} 
where $l'_{i+1}$ is defined in Proposition~\ref{prop:applystb}. Since $(\mu_k)_{k<K}$ is concave and since by definition of the concave grid $\mathcal S^{\log}_{l_i, r_i}$ we have that for $l'_{i+1} \neq l_i$, 
$$\mu_{l'_{i+1}} - \mu_{l_{i+1}} \geq \frac{\epsilon_i}{4}\;.$$
However this would imply $$\mu_{l_i} < \tau_i - \frac{\epsilon_i}{8} - \frac{\epsilon_i}{4} < \tau - \epsilon_i\;,$$ contradicting $\xi_i$, hence  $l_i = l'_{i+1}$ and therefore via choice of $l'_{i+1}$, $l_i + 1 = l_{i+1}$. Therefore as $\mu_{k<K}$ is concave, 
$$\forall k < l_{i+1}, \mu_k \leq \mu_{l_{i+1}}\;. $$
The property $\mu_{l_{i+1}} \geq \tau - \epsilon_{i+1}$ follows directly from (4), we have $\xi^{(L)}_{i+1}$
\paragraph{Case 2} 
Consider the case where $\stb$ outputs $l_{i+1}$ such that,
$$|\mu_{l_{i+1}}-\tau_i| \leq \epsilon_i/8.$$
From Lemma \ref{lem:con_analytic_1} we have that the sequence $(\mu_k)_{k<K}$ is increasing  on $[\tau_i - \frac{1}{8}\epsilon_i, \tau_i + \frac{1}{8}\epsilon_i]$ Therefore $\forall k < l_{i+1}, \mu_k \leq \mu_{l_{i+1}}$. Hence $\xi^{L}_{i+1}$ holds. \\\par

And so we have as desired that
$$\xi_{i+1}^{(L)} \cap \xi_i \cap \mathcal E_i \subset \{|\mu_{l_{i+1}}-\tau_i| \leq \epsilon_i/8~~OR~~\mu_{l'_{i+1}}+ \epsilon_i/8 < \tau_i < \mu_{l_{i+1}}- \epsilon_i/8\} \cap \xi_i \cap \mathcal E_i.$$
This concludes the proof.

\end{proof}

\begin{proposition}\label{prop:xiA}
We have that for $i \leq M$
$$\mathbb P\Big(\xi_{i+1}^{(A)}\Big| \xi_i \cap \mathcal E_i^c\Big) = 1.$$
\end{proposition}
\begin{proof}
On $\xi_i \cap \mathcal E_i^c$, we know that $m_i = \lfloor \frac{l_i+r_i}{2}\rfloor$ and
$$\mu_{m_i} \leq \tau_i +\frac{1}{8} \epsilon_i = \tau - \frac{5}{8} \epsilon_i,$$
and
$$\mu_{l_i}\lor \mu_{r_i} \geq \tau - \epsilon_i,$$
and so by Lemma~\ref{lem:con_analyitc2} we conclude that for any $k\leq K$, $\mu_k <\tau - \frac{1}{8} \epsilon_i$. And so $\xi_{i+1}^{(A)}$ holds.

\end{proof}

\begin{corollary}\label{cor:xirec}
We have that
$$\mathbb P(\xi_{i+1}| \xi_i ) \geq 1 - 2\delta'_i$$
\end{corollary}
\begin{proof}
This holds by combining Propositions~\ref{prop:xiLR} and Proposition~\ref{prop:xiA}.
\end{proof}

Hence by Corollary~\ref{cor:xirec} and for any $I \leq M$ we have, 
$$\mathbb P(\cap_{i \leq I}\xi_i) \geq \prod_{i\leq I} (1 - 2\delta'_i)\geq 1 - 2\sum_{i=1}^I \delta'_i.$$

For $I, i \leq M$ consider the event 

\begin{equation}\label{eq:con_eta}
\begin{split}
\eta_i^I := \Bigg\{\;|\hmu_{m,i} - \mu_{m_i}| \lor |\hmu_{l,i} - \mu_{l_i}| \lor |\hmu_{r,i} - \mu_{r_i}| \lor \\|\hmu_{l-1, i} - \mu_{l_i-1}| \lor |\hmu_{r+1,i} - \mu_{r_i+1}| \leq \frac{1}{16}\epsilon_i \lor \epsilon_I\Bigg\},
\end{split} 
\end{equation} 
which via Azuma's martingale inequality occurs with probability greater than, 
\begin{equation}\label{eq:prob_eta} 
1- 10\exp\left(- \frac{1}{2}T_2^{(i)} \epsilon_i^2\right) \geq 1- 10\delta_i.
\end{equation}

\begin{proposition}\label{prop:containxi}
Fix $I \leq M$ and assume that there exists $k$ such that $\mu_k > \tau - \frac{1}{8}\epsilon_I$. On $\xi_I$, we have that $\{k : \mu_k \geq \tau \} \subset \{l_I, \ldots, r_I\} \subset \{k : \mu_k \geq \tau - \epsilon_I\}$.
\end{proposition}

\begin{proof}
First note that under the condition $\mu_k > \tau - \frac{1}{8}\epsilon_I$ we have that $\xi_I^{(L)} \cap \xi_I^{(R)}$ holds. Therefore the second inclusion holds, see Corollary~\ref{cor:xirec} and the definition of $\xi_I$. Now assume $\{k : \mu_k = \tau \} \neq \emptyset$. Let $k^*$ be as in the proof of Lemma~\ref{lem:con_analytic_1}. By definition of $\xi_I$ and since $(\mu_k)_k$ is concave, it is clear that $l_I \leq k^* \leq r_I$. The first inclusion then follows again by definition of $\xi_I$. In the case where $\{k : \mu_k = \tau \} = \emptyset$ the first inclusion is obvious. 
\end{proof}

\begin{proposition}\label{prop:empty}
Fix $I \leq M$ and assume that for all $k$, $\mu_k \leq  \tau - \frac{1}{8}\epsilon_I$. On $\xi_I \cap (\cap_{i\leq M}\eta_i^I)$, we have that $\hat S = \emptyset$.
\end{proposition}
\begin{proof}
Under the conditions of the proposition we have that $\mu_{m_i} \leq  \tau - \frac{1}{8}\epsilon_I$, for all $i$ and this implies the result by definition of the $\eta_i^I$ and $\mathcal I_m$.
\end{proof}

\begin{proposition}\label{prop:segm}
Fix $I \leq M$. On $\xi_I \cap \Big(\cap_{i\leq M} \eta_i^I\Big)$, we have that
$$ \mathcal I_m \subset \{l_I, \ldots, r_I\},$$
and also
$$l_I \in \mathcal I_l~~~r_I\in \mathcal I_r.$$
\end{proposition}
\begin{proof}
On $\cap_{i\leq I} \eta_i^I$, we have that $\mathcal I_m \subset \{k : \mu_k \geq \tau \} \cup \{l_I, \ldots, r_I\}$, and so from Propositions~ and~\ref{prop:containxi} and~\ref{prop:empty}, we have on $\cap_{i \leq I} \eta_i^I \cap \xi_I$, that $\mathcal I_m \subset \{l_I, \ldots, r_I\}$.

The proof that $l_I \in \mathcal I_l$ on $\xi_I\cap \eta_I^I$ - as well as the fact that $r_I\in \mathcal I_r$ - follows immediately by combining the definition of $\mathcal I_l$ - resp.~$\mathcal I_r$ - with Proposition~\ref{prop:containxi} and~\ref{prop:empty}, and the definition of $\eta_I^I$.
\end{proof}

\begin{proposition}\label{prop:segmlow}
Fix $I \leq M$, and assume that $m_I \not\in \mathcal I_m$. On $\xi_I\cap \Big(\cap_{i \leq I} \eta_i^I\Big)$, we have that $\{k : \mu_k \geq \tau + 4 \epsilon_i \} \subset \emptyset \subset \{\hat l, \ldots, \hat r\} \subset \{k : \mu_k \geq \tau - \epsilon_I \}$.
\end{proposition}
\begin{proof}
On $\xi_I \cap \Big(\cap_{i \leq I} \eta_i^I\Big)$ we have from Proposition~\ref{prop:segm} that 
$\mathcal I_m \subset \{l_I, \ldots, r_I\}$ and that $l_I \in \mathcal I_l, r_I \in \mathcal I_r$.  This implies that on $\xi_I \cap \Big(\cap_{i \leq I} \eta_i^I\Big)$, $\{\hat l, \ldots, \hat r\} \subset \{l_I, \ldots, r_I\}$. Together with Propositions~\ref{prop:containxi} and~\ref{prop:empty} this implies that on $\xi_I \cap \Big(\cap_{i \leq I} \eta_i^I\Big)$ we have $\{\hat l, \ldots, \hat r\} \subset \{k : \mu_k \geq \tau - \epsilon_I\}$.

Moreover, on $\eta_I^I$, we have by the assumption of Proposition~\ref{prop:segmlow} that $\mu_{m_I} \leq \tau +\frac{17}{8} \epsilon_I$. Together with Proposition~\ref{prop:containxi} and~\ref{prop:empty} and Lemma~\ref{lem:con_analyitc2}, this implies that on $\xi_I\cap \eta_i^I$, $\forall k \leq K, \mu_k \leq \tau + 4\epsilon_I$. This concludes the proof with the fact that $\{\hat l, \ldots, \hat r\} \subset \{l_I, \ldots, r_I\}$.
\end{proof}

\begin{proposition}\label{prop:segmhigh}
Fix $I \leq M$, and assume that $m_I \in \mathcal I_m$. On $\Big(\cap_{i \leq I} \xi_i\Big)\cap \Big(\cap_{i \leq M} \eta_i^I\Big)$, we have that $\{k : \mu_k \geq \tau + \epsilon_I \} \subset \{\hat l, \ldots, \hat r\} \subset \{k : \mu_k \geq \tau - \epsilon_I\}$.
\end{proposition}

\begin{proof}
As in the proof of Proposition~\ref{prop:segmlow}, we have on $\xi_I \cap \Big(\cap_{i \leq I} \eta_i^I\Big)$ that it holds that $\{\hat l, \ldots, \hat r\} \subset \{k : \mu_k \geq \tau - \epsilon_I\}$. Under the event $\eta_I^I$ as  $\hat l \in \mathcal I_l,\hat r \in \mathcal I_r$ we have that, 
$$\mu_{\hl -1} < \tau + \epsilon_I \; \& \; \mu_{\hr+1} < \tau + \epsilon_I.$$ 

Moreover, on $\eta_I^I$, we have by the assumption of Proposition~\ref{prop:segmhigh} that $\mu_{m_I} \geq \tau +\frac{15}{8} \epsilon_I$. Therefore, as $\mu_{m_I} \in \{\hl -1, \ldots, \hr + 1\}$  via the concavity of $(\mu_k)_{k<K}$ we have that $\{k : \mu_k \geq \tau +\epsilon_I\} \subset \{\hat l, \ldots, \hat r\}$. This concludes the proof.
\end{proof}

\begin{proof}[Proof of Proposition \ref{prop:exp_con_hi}]
Let $I\leq M$. Combining Propositions~\ref{prop:segmlow} and~\ref{prop:segmhigh}, we have on $\Big(\cap_{i \leq I} \xi_i\Big)\cap \Big(\cap_{i \leq M} \eta_i^I\Big)$ that
$$\{k : \mu_k \geq \tau + 4\epsilon_I \} \subset \{\hl, \ldots, \hr\} \subset \{k : \mu_k \geq \tau - \epsilon_I\}.$$
Note that
$$\mathbb P\Bigg[\Big(\cap_{i \leq I} \xi_i\Big)\cap \Big(\cap_{i \leq M} \eta_i^I\Big)\Bigg] \geq 1  - 10\sum_{i \leq I} \delta_i - \sum_{i \leq I} \delta_i' - (M-I) \delta_I.$$
We have by definition of $\delta_i',T_2^{(i)}$ that 
$$\delta_i' \leq \min\Bigg( \frac{1}{\log(K)^{3/4}},\, 72 \log\log(K) \delta_i^2\Bigg),$$
and also we have that $\delta_i = 2^{i - M}$
so that whenever $M-i \geq \log\log\log(K)$, we have that $\log\log(K) \delta_i^2\leq \delta_i$. And so
$$\sum_{i\leq I}\delta_i' \leq  144 \delta_i,$$
since when $M-i \geq \log\log\log(K)$, we have $72\delta_i \geq  \frac{1}{\log(K)^{3/4}}$.
And so
$$\mathbb P\Bigg[\Big(\cap_{i \leq I} \xi_i\Big)\cap \Big(\cap_{i \leq M} \eta_i^I\Big)\Bigg] \geq 1  - 164 \delta_I - (M-I) \delta_I = 1-  (M-I+164)2^{I-M}.$$    
Thus for any $i\in \{0, \ldots, M\}$ we have
$$\mathbb P\Bigg[R_T\geq 4 \Big(\frac{7}{8}\Big)^{M - i}\Bigg] \leq (i+164)2^{-i} \leq 200\Big(\frac{2}{3}\Big)^{i}.$$
This concludes the proof by summing over $I$ for finding the expected regret, and noting that there exists a universal constant $C>0$ such that $ \Big(\frac{7}{8}\Big)^{M} = \epsilon_M \leq C \sqrt{\frac{\log\log K}{T}}$, by definition of $M$.
\end{proof}

\section{Extension of results to fixed confidence setting}\label{sec:fixed_con}

\paragraph{Fixed confidence setting.}In this section we extend our results to the fixed confidence setting for the \stbp and \tbp. In this case, we define $\delta,\epsilon>0$, to be respectively the target confidence, and target precision of our algorithm. We say that a strategy $\pi$ is $(\epsilon,\delta)$-PAC if it stops sampling at some stopping time $\sT^{\pi}_{\epsilon,\delta}$ of its choice, and satisfies that with probability larger than $1-\delta$, $R_{T}^{\unu,\pi} \leq \epsilon$. In this setting the aim is to find a $(\epsilon,\delta)$-PAC strategy that minimises the expected stopping time $\EE_{\unu}[\sT^\pi_{\epsilon,\delta}].$
The following Corollaries are an immediate consequence of our previous results, thus we omit proofs.
\subsection{Lower Bounds}
The following corollary is a direct extension to Proposition \ref{prop:unstr_lo} which provides a lower bound in the unstructured case.
\begin{corollary}\label{cor:tbppaclb}
Let $\epsilon, \delta>0$. It holds that for any strategy $\pi$ that stops at a stopping time $\sT^\pi_{\epsilon, \delta}$ and that is $(\epsilon, \delta)$-PAC, there exists a unstructured bandit problem $\unu \in \B$, such that
\[
\mathbb E_{\unu} [\sT^\pi_{\epsilon, \delta}] \geq \frac{2\sigma^2 K \max(\log(K),2) (1 - K^{-1} - \delta)^2}{\epsilon^2}\,.
\]
\end{corollary}

\begin{proof}
Consider the notations of the proof of Proposition~\ref{prop:unstr_lo}. Assume that there exists an $(\epsilon, \delta)$-PAC strategy $\pi$ such that for all $Q\in \{-1,1\}^K$, we have
\[\EE_Q[\hat T^\pi_{\epsilon,\delta}] < \frac{2\sigma^2 K \max(\log(K),2) (1 - 1/K - \delta)^2}{\epsilon^2}\,.\]
From the proof of Proposition~\ref{prop:unstr_lo} it holds 
\[
\frac{1}{2^K} \sum_{Q} \PP_Q(\hat Q = Q) \leq 1/K + \sqrt{\sup_{Q’ \in \{-1,1\}^K} \EE_Q[\hat T^\pi_{\epsilon,\delta}] \epsilon^2/(2K\sigma^2 \max( \log( K),2))}\,.
\]
And so  there is a contradiction:
\[\inf_{Q} \PP_Q(\hat Q = Q) < 1- \delta\,.\]
\end{proof}

Combining this result with the lower bound from Theorem 2 of~\cite{chen2014combinatorial}, we obtain that for any $(\epsilon, \delta)$-PAC strategy, there exists a bandit problem where all arms are $1/4$-sub-Gaussian and such that the expected stopping time is of higher order than $\frac{K\log(K/\delta)}{\epsilon^2},$ 
since they prove that the expected stopping time for any $(\epsilon, \delta)$-PAC strategy is higher than $\frac{K\log(1/\delta)}{\epsilon^2},$ 
on some bandit problem.\\

The following corollary is a direct extension to Proposition \ref{prop:mon_lo} which provides a lower bound in the monotone case.
\begin{corollary}\label{cor:stppaclb}
Let $\epsilon, \delta>0$ and $K \geq 2 $. It holds that for any strategy $\pi$ that stops at a stopping time $\sT_{\epsilon, \delta}$ and that is $(\epsilon, \delta)$-PAC, there exists a unstructured bandit problem $\unu \in \Bs$, such that
\[
\mathbb E_{\unu} [\sT_{\epsilon, \delta}] \geq \frac{2\sigma^2 \max(2,\log(K)) (1 - K^{-1} - \delta)^2}{\epsilon^2}\,.
\]
\end{corollary}
A very similar result was already obtained in~\cite{karp2007noisy} , but for Bernoulli random variables in the lower bound, and without providing an explicit dependence on $\delta$. In the paper~\cite{ben2008bayesian}, they refine this bound in the case of fixed probability of error which implies that for any strategy that $(\epsilon, \delta)$-PAC, there exists a structured bandit problem where all arms are $1/4$-sub-Gaussian and such that the expected stopping time is of higher order than $(1-\delta)\log(K)/\epsilon^2$ \textit{up to terms that are negligible with respect to $\log(K)/\epsilon^2$} - which is essentially the same as what we have.

We say that a strategy is optimal if its expected simple regret (or its expected stopping time for the fixed confidence setting) matches one of this lower bounds \emph{up to a universal constant}.

\subsection{Upper Bounds}
The following Corollary is a direct extension to Proposition \ref{prop:unstr_exp_up}, which provides an upper bound on regret of the \unif algorithm. 
\begin{corollary}\label{cor:tbppac}
Let $\epsilon, \delta >0$. For any unstructured bandit problem $\unu \in \cB$, Algorithm \unif launched with parameter $T := \lfloor\frac{2 \sigma^2 K \log(2K/\delta)}{\epsilon^2}\rfloor +K$ is $(\epsilon, \delta)$-PAC.
\end{corollary}
Interestingly the stopping time can be taken here as deterministic, and this matches up to a multiplicative constant the lower bound in Corollary~\ref{cor:tbppaclb} combined with the one in~\cite{chen2014combinatorial}.\\

The following Corollary is a direct extension to Corollary\ref{cor:mo_exp_up} which provides an upper bound on the regret of the \stb algorithm,
\begin{corollary}\label{cor:stppac}
Let $\epsilon, \delta >0$. For any problem $\unu \in \mathcal{B}_s$, algorithm \stb launched with parameter $T := \lfloor \frac{21\sigma^2 \log(K)}{\epsilon^2} + 12\log(K)\rfloor$ if $\delta \geq K^{-3/4}$ and $T:=\lfloor \frac{432 \sigma^2 \log(K) \log(9/\delta)}{\epsilon^2} + 12\log(K)\rfloor$ otherwise, is $(\epsilon, \delta)$-PAC.
\end{corollary}
Interestingly, the stopping time can be taken here as constant. For $\delta$ large enough i.e.~$\delta \geq K^{-3/4}$, yet smaller than any universal constant strictly smaller than $1$, this is order optimal up to a multiplicative constant - see Corollary~\ref{cor:stppaclb}. For $\delta$ smaller, this is order optimal up to a multiplicative constant that depends on $\delta$ - and it is an open question to obtain optimality in this case.

Similar results can be obtained in \utbp and \ctbp.

\section{Supplementary discussion concerning the \tbp and \stbp}\label{app:litreview}

\subsection{Comparison of \tbp and \stbp and focus on the main difference coming from the monotone structure}\label{app:tbp_vs_mtbp}

In the \tbp, the proof of the bound of algorithm \unif is very classical. It is, as usual in bandits, event based. We consider the event where all arms concentrate around their mean with error bounded by $O(\sqrt{K\log(K/\delta)/T})$ - where the $\log(K/\delta)$ term comes from a union bound over all $K$ arms - and prove that on this event the regret is bounded. The lower bound is slightly less classical when it comes to the bandit literature, and is close in spirit to the use of a sequential version of Fano's inequality - stating effectively that the union bound in the analysis of the event on the means is tight.

In the \stbp, however, both the algorithm \stb and its proof are far less classical. As discussed in Section \ref{sec:intro} a naive, yet suboptimal, approach to the \stbp is a binary search. At each step we sample an arm $O(T/\log(K))$ times and then decide to go left or right. This kind of strategy relies on making a correct decision at each step, and requires an event based analysis. The event is here that all $O(\log(K))$ sampled arms have their empirical means that concentrate around the true means at rate $\sqrt{\log(K)\log(\log(K)/\delta)/T}$ - the $\log(\log(K)/\delta)$ term coming from the union bound.  This results in a regret of order $\sqrt{\log(K)\log(\log(K))/T}$, which is strictly sub-optimal.
With this in mind we consider a different algorithm that performs a `corrective' version of the binary search, i.e.~a version where the algorithm can self-correct if it realises that it made a mistake This subtle, yet fundamental difference highlights the very big gap between \tbp and \stbp.

\subsection{Supplementary details of the related works: \tbp}
Comparing \tbp and \stbp thoroughly to related work is tricky since many related works are written in the fixed confidence setting. We extend the discussion here with respect to what is done in the paper.\\

In the \textit{problem independent regime} of the \tbp, current state of the art results can be deduced from the paper~\cite{locatelli2016optimal}. A corollary to the lower bound in~\cite{locatelli2016optimal} in the problem independent case is that for any algorithm, there exists a bandit problem where all arms have their distribution on $[0,1]$ and such that with probability larger than $1/2$, at least one arm is missclassified and at more than a strictly positive constant times $\sqrt{K/T}$ from the threshold - this is also a corollary from the lower bound in~\cite{bubeck2009pure} for the different problem of best arm identification. Reciprocally, the state of the art upper bound in the problem independent case is a corollary to the upper bound in~\cite{locatelli2016optimal}. In the problem independent setting, with probability larger than $1-\delta$, all arms are within a strictly positive constant times $\sqrt{K\log(K \log T / \delta)/T}$ from $\tau$. As one can see, current state of the art upper and lower bounds are are far from matching in the \textit{problem independent case}.

\subsection{Supplementary details of the related works: \stbp}
 The papers \cite{feige1994computing}, \citet{ben2008bayesian} and \citet{emamjomeh2016deterministic} introduce a noisy binary search \emph{with corrections}. However in the above papers the probability of making an error during the binary search is treated as fixed. But this assumption does not hold in the setting of the \stbp. In \cite{Nowak09binary} a more generalised version of the binary search is considered with weaker assumptions on structure, however there is no contribution to classical binary search beyond that of \cite{karp2007noisy}.

\citet{karp2007noisy} consider the special case where all arms $k$ follows a Bernoulli distribution with parameter $p_k$ and $p_1 < ... < p_K$, and the aim is to find a $i$ such that $p_i$ is close to $1/2$. In the \textit{fixed confidence setting}, they prove that the naive binary search approach is not optimal and propose an involved exponential weight algorithm, as well as a random walk binary search, for solving the problem. They prove that for $\epsilon,\delta>0$ fixed, then the algorithm returns all arms above threshold with probability larger than $1-\delta$ and tolerance $\epsilon$ in an expected number of pulls less than a multiplicative constant \textit{that depends on $\delta$ in a non-specified way} times $\log_2(K)/\epsilon^2$. They prove that this is optimal up to a constant depending on $\delta$. In the paper~\cite{ben2008bayesian} they refine the dependence in $\delta$ in a slightly different setting - where one has a fixed error probability. They prove that \textit{up to terms that are negligible with respect to $\log(K)/\epsilon^2$}, a lower bound in the expected stopping time is of order $(1-\delta)\log(K)/\epsilon^2$.

\subsection{Contribution with respect to the literature}\label{ss:compa}

Our contributions can be summarised are as follows: 

\begin{itemize}
\item \textit{Problem independent optimal rate for \tbp} We provide the first -to the best of our knowledge - upper and lower bounds in the \textit{problem independent regime} for the \tbp - both in the fixed confidence and fixed budget setting - as well as an associated parameter-free algorithm, \unif.

\item\textit{Extension of \stbp to $\sigma^2$-sub-Gaussian distribution}
The lower bound and optimal algorithm proposed in \cite{karp2007noisy} is specific to the assumption that all arms follow a Bernoulli distribution - and related literature makes even more constraining assumptions~\cite{feige1994computing, ben2008bayesian, emamjomeh2016deterministic}. An extension of their algorithms- even in the fixed confidence setting - beyond this assumption is non-trivial. We propose an algorithm whose only assumption is that the arms follow a $\sigma^2$-sub-Gaussian distribution.

\item\textit{\stbp in the fixed budget setting}
We treat in a problem independent optimal way the \textit{fixed budget setting}. 

The algorithms proposed in \citet{karp2007noisy} - as well as in~\cite{feige1994computing, ben2008bayesian, emamjomeh2016deterministic} in a more restricted setting regarding the error distributions  - operate in the fixed confidence setting. Adapting their results to a fixed budget setting is challenging, in particular since we consider the \textit{expected maximal gap} as a measure of performance - see Section~\ref{sec:prob}. 

 \item\textit{Simultaneous bound on all probability} The \stb regret bound holds simultaneously across all probabilities. That is for all $\delta > 0$ and after $T$ rounds of our algorithm, we have a guarantee that with probability larger than $1- \delta$, the simple regret will be bounded depending on $\delta$. This is in strong contrast to what is done in the fixed confidence literature ~\cite{karp2007noisy, ben2008bayesian, emamjomeh2016deterministic, chen2014combinatorial}, where $\delta$ is given as a parameter to the algorithm, and where the behaviour of the algorithm is only studied on an event of probability $1 - \delta$, and a clear improvement with respect to~\cite{karp2007noisy} where the dependence in $\delta$ is not explicitly stated in the bound on regret. Our result is more general, as it allows us to get a bound on the \textit{expected simple regret} for the fixed budget setting, but also to easily transform our algorithm to the fixed confidence setting.

\end{itemize}
We also refer to Table~\ref{tab:KV3} for a comprehensive summary of state of the art rates, as well as of our rates.
\subsection{Problem dependent regime}\label{sec:probdep}

While not the focus of this paper we comment on the performance of our algorithms in the problem dependent regime for the \tbp and \stbp. The problem dependent regime is defined as follows: for some sequence $\Delta \in \mathbb{R}_+^{K}$ we consider a sub class of problems $\mathcal{B}^\Delta \subset \cB$ where
\[\cB^\Delta = \{\nu \in \cB : \forall k \in [K], \; |\mu_k - \tau| = \Delta_k\}\;.\]

Similarly we can define
\[\Bs^\Delta = \{\nu \in \Bs : \forall k \in [K], \; |\mu_k - \tau| = \Delta_k\}\;.\]

The mechanics of the game are then identical to those described in Section \ref{sec:prob} with the exception that we consider a modified version the simple regret
\[\tilde R_{T}^{\unu,\pi} = \PP_{\unu} \!\left(\exists k \in [K] : \hQ_k^\pi \neq Q_k\right),\]
that is, the probability the learner makes at least one miss classification - which is more relevant than the simple regret considered in this paper in the regime where the $\Delta_k$ are not very small, depending on $T,K$. 

In the case of the \tbp consider the class of problems $\cB^\Delta$ for  some $\Delta \in \mathbb{R}_+^{K}$. An upper bound on the simple regret of the order $\exp\left(-c\frac{1}{K}\sum \Delta_i^{2} \frac{T}{K} + c'\log(\log(T)K) \right)$ is provided from  \cite{locatelli2016optimal}, for the APT algorithm that does not take any parameters - where $c,c'>0$ are universal constants. A matching lower bound is also provided in~\cite{locatelli2016optimal}, up to universal constants in the exponential. In the same setting we can upper bound the simple regret of the \unif algorithm by $\sum_k \exp\left(-c\Delta_k^2 \frac{T}{K}\right)$, where $c>0$ is a universal constant. Clearly the uniform algorithm under performs heavily in cases with high variance across the gaps, this should not come as a surprise. \\

In the case of the \stbp consider the class of problems $\Bs^\Delta$ for  some $\Delta \in \mathbb{R}_+^{K}$. We can construct and immediate lower bound on the simple regret of the order $\exp\left(-cT \min_{k\in[K]}{\Delta^2_k}\right)$ - where $c>0$ is some universal constant - while the \stb algorithm achieves an upper bound of the order $\exp\left(-c\frac{T}{\log(K)}\min_{k\in [K]}\Delta_k^2\right)$ - where $c>0$ is some (different) universal constant. Thus, while it is not optimal, the algorithm \stb is nevertheless quite efficient in the problem dependent setting.

\begin{table*}
\def\arraystretch{1.85}
 \begin{tabular}{c||c|c|c|c}
 {} & \multicolumn{2}{c|}{State of the art} & \multicolumn{2}{c}{Our results}\\
  \hline  
{} & LB & UB & LB & UB \\ 
 \hline\hline
  $\displaystyle \underset{\text{\citep{locatelli2016optimal}}}{\text{\tbp FB \textsuperscript{\hyperref[fnote:1]{1}} }} $ & $\sqrt{\frac{K}{T}}$ & $\sqrt{\frac{K\log(K\log T )}{T}}$  & $\sqrt{\frac{K\log(K)}{T}}$\textsuperscript{\hyperref[fnote:4]{4}} & $\sqrt{\frac{K\log(K)}{T}}$\textsuperscript{\hyperref[fnote:5]{5}}  \\

  $\displaystyle \underset{\text{\citep{chen2014combinatorial}}}{\text{\tbp FC}} $& $\frac{K \log(\delta^{-1})}{\epsilon^2}$ & $\frac{K \log(K^2\epsilon^{-2}\delta^{-1})}{\epsilon^2}$   & $\frac{K\log(K) (1 - K^{-1}- \delta)}{\epsilon^2}$\textsuperscript{\hyperref[fnote:6]{6}} & $\frac{K\log(K \delta^{-1}) }{\epsilon^2}$ \\
  \hline
 \stbp FB &  None & None & $\sqrt{\frac{\log(K)}{T}}$ & $\sqrt{\frac{\log(K)}{T}}$  \\

  $\displaystyle \underset{\text{\citep{karp2007noisy} }}{\text{\stbp FC \textsuperscript{\hyperref[fnote:2]{2}} }}$ & $\frac{\underline{c}_\delta \log(K)}{\epsilon^2} $ \textsuperscript{\hyperref[fnote:3]{3}}  & $\frac{\bar c_\delta\log(K)}{\epsilon^2}$  & $\frac{ (1-K^{-1} - \delta)\log(K)}{\epsilon^2}$\textsuperscript{\hyperref[fnote:7]{7}} & $\frac{\log(K) \log(\delta^{-1})}{\epsilon^2}$\textsuperscript{\hyperref[fnote:8]{8}}\\
\end{tabular}
 \caption{Upper and lower bounds on the expected simple regret in the fixed budget (FB) setting and on the expected stopping time for $(\epsilon, \delta)$-PAC strategies in the fixed confidence (FC) setting. All results are given up to universal multiplicative constant - in the case where the sub-Gaussian parameter $\sigma$ is set to $1$. \textit{Left:} previous state of the art bounds. \textit{Right:} bounds from our paper.
 \label{tab:KV3}}
 \end{table*}

\section{Supplementary discussion}\label{sec:diss}

\subsection{Parameters of the algorithms}
The \unif algorithm only takes $T$ as a parameter, see Subsection~\ref{ss:anytime} for a discussion on how to make it anytime. The \stb algorithm takes only $\sigma,K,T$ as parameters. Again, see Subsection~\ref{ss:anytime} for an anytime version. Getting rid of $\sigma$ is however more tricky and is an open problem. We believe that in some pathological situations, the knowledge of $\sigma$ is necessary. Note however that it is a very mild assumption. Indeed $\sigma$ comes from Definition~\ref{def:subgau}.
In many case, natural choices for $\sigma$ are available - for instance if reward are bounded. Regarding \UTB and \CTB, simple extensions can be made so that they also consider the sub-Gaussian case.

\stepcounter{footnote}\footnotetext{See also \citet{bubeck2009pure} for the LB.}\label{fnote:1}
\stepcounter{footnote}\footnotetext{Here $\underline{c}_\delta, \bar{c}_\delta>0$ is a function of $\delta$ that is left unspecified in \citet{karp2007noisy}.}\label{fnote:2}
\stepcounter{footnote}\footnotetext{See also \citet{ben2008bayesian} for the LB  $\frac{(1-\delta)\log(K)}{\epsilon^2}$ up to terms that are negligible with respect to $\log(K)/\epsilon^2$.}\label{fnote:3}
\stepcounter{footnote}\footnotetext{In \citet{locatelli2016optimal} The problem complexity $H$ is upper bounded by $K/\epsilon^2$. Replacing $H$ with such provides the given upper bound
}\label{fnote:4}
\stepcounter{footnote}\footnotetext{The lower bound is well known, see \citet{bubeck2009pure}.
}\label{fnote:5}
\stepcounter{footnote}\footnotetext{And combining this with the lower bound in~\cite{chen2014combinatorial}, we get the problem independent lower bound of order $\frac{K\log(K \delta^{-1}) }{\epsilon^2}$ that matches our upper bound.}\label{fnote:6}
\stepcounter{footnote}\footnotetext{See also~\cite{ben2008bayesian} for a LB that is essentially equivalent to this.}\label{fnote:7}
\stepcounter{footnote}\footnotetext{In the case where $\delta \geq K^{-3/4}$ and is smaller than any universal constant strictly smaller than $1$, our UB is more refined and of order $\frac{\log(K) }{\epsilon^2}$, which is order optimal.}\label{fnote:8}

\subsection{Making the algorithms anytime} \label{ss:anytime}
Although the \unif algorithm, for simplicity, takes a known budget $T$ it can trivially be extended to an anytime algorithm. With $T$ unknown one can easily obtain a uniform distribution of pulls by repeatedly pulling all arms once in a batch until the ``unknown" budget is expended.\par

In the case of the \stb Algorithm such a trivial extension is not possible. At each time step the number of times the arms in the current node are pulled is dependant upon budget $T$. Now note that it is possible to apply a doubling trick to our problem. I.e.~first call the algorithm \stb with budget $T = \lfloor 6\log(K) \rfloor +1$, and then until the algorithm is stopped, always double the budget and call algorithm \stb from scratch. Then when the algorithm is stopped, recommend the arm recommended by the last full iteration. Note that this arm will have been selected with at least a fourth of the budget, and so Proposition~\ref{prop:mo_hiprob_up} and Corollary~\ref{cor:mo_exp_up} hold with the doubling trick and therefore without taking $T$ as parameter, and replacing $T$ by $T/4$ in the bound. Similar tricks hold also for \UTB and \CTB.

\subsection{Computational complexity}

The computational complexity of both our algorithms is very low. Algorithm \unif is just uniform sampling, and then a computation of $K$ empirical means and their comparison to the threshold. I.e.~this is in total $n$ operations (where by operations we mean addition or comparisons), and needs to store only $K$ variables, i.e.~the empirical means.

Algorithm \stb consists of
\begin{itemize}
    \item first running Algorithm \explore, which consists just in computing about $\log(K)$ empirical means, and taking decisions based on them. The algorithm just needs to perform $n$ operations (where by operations we mean addition or comparisons), and needs to store only about $\log K$ variables, i.e.~the empirical means and position of sampled arms.
    \item then running Algorithm \choose which consists in scanning one time the list of sampled arms, i.e.~doing about $\log(K)$ operations, and returning the median. The number of operations is therefore of order $\log(K)$ and the algorithm needs to store only about $\log(K)$ variables, i.e.~the empirical means and position of relevant sampled arms.
\end{itemize}
Similarly, the computational complexity of \UTB and \CTB is also low.

\end{document}